\documentclass[journal]{IEEEtran}
\IEEEoverridecommandlockouts
\usepackage{cite}
\usepackage{algorithm}
\usepackage{algpseudocode}
\usepackage{graphicx}
\usepackage{textcomp}
\usepackage{xcolor}
\usepackage{bbm}
\def\BibTeX{{\rm B\kern-.05em{\sc i\kern-.025em b}\kern-.08em
T\kern-.1667em\lower.7ex\hbox{E}\kern-.125emX}}
\usepackage{afterpage}
\usepackage{mathtools}
\usepackage{cuted}
\usepackage{subfigure}
\usepackage{array}
\usepackage{epsf}
\usepackage{times}
\usepackage{epsfig}
\usepackage{epstopdf}
\usepackage{hyperref}
\usepackage{mathtools}
\usepackage{amsthm}
\usepackage{amssymb}
\usepackage{comment}
\usepackage{siunitx}
\usepackage{lipsum}
\usepackage{kantlipsum}
\usepackage{dblfloatfix}
\usepackage{adjustbox}
\usepackage{makecell}
\usepackage[font=small]{caption}
\usepackage{subcaption}
\usepackage{amsmath}
\usepackage{amsfonts}
\usepackage{mathtools}
\usepackage{amscd}
\usepackage{bm}
\usepackage{subfigure}
\usepackage{float}

\usepackage{tikz}
\usepackage{pgfplotstable}
\usepackage{rotate}

\definecolor{rubblue}{cmyk}{1,0.5,0,0.6}
\definecolor{rubgreen}{cmyk}{0.5,0,1,0}
\definecolor{rubgray}{cmyk}{0.03,0.03,0.03,0.1}

\usepgfplotslibrary{units}

\usetikzlibrary{%
patterns,%
calc,%
fit,%
arrows,%
plotmarks,%
shadows,%
chains,%
shapes%
}

\tikzset{>=latex'} 
\tikzstyle{every picture}+=[remember picture] 
\pgfdeclarelayer{background}
\pgfdeclarelayer{foreground}
\pgfsetlayers{background,main,foreground}

\tikzstyle{blueblock}=[draw=rubblue, rectangle, thick, drop shadow, minimum width=20mm, minimum height=8mm,fill=rubblue!20, text width=20mm, text centered]
\tikzstyle{bluebox}=[draw=rubblue, rectangle, thick, drop shadow, minimum width=8mm, minimum height=8mm,fill=rubblue!20, text width=8mm, text centered]
\tikzstyle{greenblock}=[draw=rubgreen, rectangle, thick, drop shadow, minimum width=20mm, minimum height=8mm,fill=rubgreen!20, text width=20mm, text centered]
\tikzstyle{dot} = [draw, circle, minimum size=0.2pt,scale=0.3,fill=black,black]
\tikzstyle{smalldot} = [draw, circle, minimum size=0.1pt,scale=0.2,fill=black,black]
\tikzstyle{reddot}  =[draw,circle,minimum size=0.2pt,scale=0.8,fill=red,thin]
\tikzstyle{greendot}  =[draw,circle,minimum size=0.2pt,scale=0.8,fill=Green,thin]
\tikzstyle{bluedot}  =[draw,circle,minimum size=0.2pt,scale=0.8,fill=blue,thin]
\tikzstyle{whitedot}=[draw,circle,minimum size=0.2pt,scale=0.8,fill=white,thin]
\tikzstyle{blackdot} = [draw, circle, minimum size=0.2pt,scale=0.7,fill=black,black]
\tikzstyle{sum} = [drop shadow, draw=rubblue, thick, fill=rubblue!20, circle]
\tikzstyle{relay} = [blueblock, minimum width=5mm, minimum height=20mm, text width=5mm, rounded corners=2pt]
\tikzstyle{relay2} = [blueblock, minimum width=5mm, minimum height=15mm, text width=5mm, rounded corners=2pt]
\tikzstyle{relay3} = [blueblock, minimum width=5mm, minimum height=25mm, text width=5mm, rounded corners=2pt]
\tikzstyle{relay4} = [blueblock, minimum width=5mm, minimum height=10mm, text width=5mm, rounded corners=2pt]
\tikzstyle{relay5} = [blueblock, minimum width=5mm, minimum height=50mm, text width=5mm, rounded corners=2pt]
\tikzstyle{relay6} = [blueblock, minimum width=5mm, minimum height=5mm, text width=5mm, rounded corners=2pt]
\tikzstyle{circgreen} = [draw, circle, inner sep=2pt, fill=rubgreen, drop shadow, thick]
\tikzstyle{circwhite} = [draw, circle, inner sep=2pt, fill=white, drop shadow, thick]
\tikzstyle{circdashed} = [draw, dashed, circle, inner sep=2pt, fill=rubgray, drop shadow, thick]
\tikzstyle{vertbox} = [rectangle, draw=rubblue, thick, rotate=90, text centered, minimum width=16.5mm, minimum height=8mm, text width=16.5mm, inner sep=0pt, fill=rubblue!20, drop shadow]
\tikzstyle{vertboxb} = [rectangle, draw=rubblue, thick, rotate=90, text centered, minimum width=16.5mm, minimum height=8mm, text width=16.5mm, fill=rubblue!20, drop shadow]
\tikzstyle{vertboxshort} = [rectangle, draw=rubblue, thick, rotate=90, text centered, minimum width=10mm, minimum height=8mm, text width=10mm, inner sep=0pt, fill=rubblue!20, drop shadow]
\tikzstyle{smalldotgreen} = [draw=rubgreen, circle, minimum size=0.2pt,scale=0.8,fill=rubgreen!20]
\tikzstyle{antenna} = [regular polygon, regular polygon sides=3, draw, shape border rotate=180, minimum size=0.2pt, scale=0.3]

\tikzstyle{poly} = [regular polygon, regular polygon sides=6, shape aspect=0.5, minimum width=1.5cm, minimum height=0.35cm, draw, dashed]

\definecolor{cff9e00}{RGB}{255,158,0}
\definecolor{c4fff00}{RGB}{79,255,0}
\definecolor{cff0012}{RGB}{255,0,18}
\definecolor{c00c5ff}{RGB}{0,197,255}
\definecolor{c046f00}{RGB}{4,111,0}
\definecolor{c004b9d}{RGB}{0,75,157}

\newlength{\mylen}
\usetikzlibrary{petri}
\usetikzlibrary{shapes}
\usetikzlibrary{positioning,arrows,patterns}
\usepackage{scalefnt}
\usetikzlibrary{calc,decorations.markings}
\pgfplotsset{compat=1.10}
\usetikzlibrary{arrows.meta}
\usepackage[columnwise,switch,mathlines]{lineno} 

\usepackage{pgfplots}
\pgfplotsset{compat=newest}
\usepgfplotslibrary{groupplots}

\IEEEoverridecommandlockouts
\newtheorem{theorem}{Theorem}
\newtheorem{corollary}{Corollary}

\newtheorem{lemma}{\textbf{Lemma}}

\newtheorem{assmon}{\textbf{Assumption}}

\begin{document}
\title{A Unified Framework for Fair and Personalized Decentralized Learning under Communication Constraints}

\author{\IEEEauthorblockN{Krishnendu S. Tharakan, \IEEEmembership{Member, IEEE}, Carlo Fischione, \IEEEmembership{Fellow, IEEE}}
 \thanks{
 K. S. Tharakan and C. Fischione are with the School of Electrical Engineering and Computer Science, KTH Royal Institute of Technology, Stockholm, Sweden. Email: \{tharakan, carlofi\}@kth.se.} 
 }

\setlength{\belowdisplayskip}{2pt}
\maketitle
\begin{abstract}
Decentralized learning systems aim to collaboratively train models across multiple clients without relying on a central coordinator. While decentralization improves scalability, privacy, and robustness, it also exacerbates three fundamental challenges: statistical heterogeneity across clients, fairness in client-level performance, and stringent communication constraints. This raises a natural question: \emph{how fair can decentralized learning be under limited communication?}
We address this question by presenting a unified framework for decentralized learning under communication constraints, bringing together graph-based personalization, agnostic fairness, and compressed event-triggered communication. Specifically, we propose a new algorithm DMFL-SQ, a decentralized multi-task learning algorithm that couples personalized model training over a communication graph with an agnostic mixture fairness objective, while reducing communication through sparsification, quantization, and event-triggered synchronization. We establish convergence guarantees for general non-convex objectives
and show that DMFL-SQ achieves an
$\mathcal{O}(T^{-1/2})$ rate in expected squared Moreau-envelope
stationarity despite sparse, quantized, and event-triggered
communication. We further derive PAC-Bayes generalization guarantees for the fairness-aware mixture objective. Experiments on CIFAR-10 and the real heterogeneous MUSMET EEG dataset demonstrate that DMFL-SQ substantially reduces communication while maintaining predictive performance and improving fairness across clients. Together, our theoretical and empirical results show that personalization, fairness, and communication efficiency can be jointly achieved in decentralized learning while preserving the dominant convergence rate. 
\end{abstract}
\IEEEpeerreviewmaketitle

\section{Introduction}
\label{sec:intro}
Federated learning (FL) enables multiple clients to collaboratively train machine learning models without directly sharing their local data~\cite{mcmahan2017communication, amiri20}. In the classical FL setting, a central server coordinates the training process by collecting local model updates from participating clients and aggregating them into a single global model. This paradigm has enabled large-scale learning across distributed data sources, including mobile devices, edge sensors, and networked intelligent systems~\cite{tharakan25, krishnendu22}. However, practical deployments of FL face several fundamental challenges. Client data are often statistically heterogeneous, making a single global model suboptimal for many users. Optimizing only the average performance may also lead to unfair outcomes, where clients with scarce, noisy, or non-representative data experience disproportionately high loss~\cite{litian20}. In addition, frequent communication of high-dimensional model updates creates a severe bottleneck in bandwidth-limited, energy-constrained, and unreliable networks.

A natural way to address statistical heterogeneity is through personalized or multi-task learning. Instead of enforcing a single shared model across all clients, each client learns a personalized model while still benefiting from collaboration with related clients. In graph-based personalized learning, clients are represented as nodes of a communication or similarity graph, and neighboring models are encouraged to remain close through a graph-based regularization mechanism. This allows the learning process to exploit statistical similarity among clients while preserving local adaptability. Such formulations have been studied in multi-task and personalized federated learning~\cite{smith2017federated,NEURIPS2020_tdinh,NEURIPS2020_alireza, tharakan25}. Nevertheless, most existing personalized FL methods either rely on centralized coordination or do not explicitly address fairness across clients.

Fairness-aware federated learning has emerged as an important direction for mitigating disparities in client-level performance. Rather than optimizing only the average loss, fairness-aware approaches seek to improve the performance of underrepresented or high-loss clients. Methods based on distributionally robust optimization, agnostic federated learning, and conditional value-at-risk (CVaR) aim to control the worst-case or upper-tail client risk, thereby improving robustness to heterogeneous client distributions~\cite{li2019fair,lahoti2020fairness,hashimoto2018fairness}. However, these methods typically assume centralized aggregation and frequent global synchronization. Their integration into fully decentralized networks with limited communication remains largely unexplored.

Communication efficiency is another central challenge in decentralized learning~\cite{tang18,pmlr-v97-koloskova19a}. Unlike server-based FL, fully decentralized systems require clients to exchange information directly with their neighbors over a communication graph. This removes the need for a central coordinator and improves robustness to server failures, but it also makes communication more complex. Each client must repeatedly transmit model updates or gradients over bandwidth-limited links, and these updates are often high-dimensional. Compression, sparsification, quantization, and event-triggered communication can substantially reduce this cost~\cite{singh23,stich18}. However, aggressive communication reduction may introduce additional bias, network disagreement, and instability, especially in non-convex learning problems. Establishing rigorous convergence guarantees under decentralized
graph coupling, fairness regularization, and compressed
event-triggered communication is therefore technically challenging.

Despite substantial progress on personalization, fairness, and communication efficiency, these components are usually studied in isolation. To the best of our knowledge, existing work does not provide a unified theoretical and algorithmic framework that simultaneously captures graph-based personalization, agnostic fairness, and compressed decentralized communication under general non-convex objectives, and DMFL-SQ introduces a new integrated decentralized algorithmic framework.
This paper addresses the following question: {\textbf{\emph{how fair can decentralized learning be under limited communication?}} We answer this question by developing a unified framework for decentralized learning under communication constraints, bringing together three practical components: graph-based personalization to handle statistical heterogeneity, agnostic fairness to control client-level performance disparities, and sparsified, quantized, and event-triggered communication to reduce bandwidth consumption.

DMFL-SQ introduces a new integrated decentralized algorithmic framework. Existing compressed decentralized optimization methods, such as SPARQ-SGD~\cite{singh23}, CHOCO-SGD~\cite{pmlr-v97-koloskova19a}, and related variants, primarily target consensus optimization or a shared global model. In contrast, \textsc{DMFL-SQ} optimizes a graph-regularized personalized objective in which client models need not agree. Similarly, agnostic federated learning methods address worst-client or mixture-risk robustness, but are typically studied in centralized or server-assisted settings and do not account for sparse-quantized event-triggered neighbor communication over a decentralized graph~\cite{chenzhixiong24, chen24, mohri2019agnostic}.

In contrast, \textsc{DMFL-SQ} tackles decentralized multi-task learning with graph-based personalization and agnostic client-level fairness. This coupling creates new analytical challenges because the fairness envelope is generally nonsmooth, the optimization variable is the full personalized model collection, and communication compression introduces stale neighbor-copy errors into the graph-coupling direction. The proposed algorithm addresses these challenges by using scalar fairness-consensus to identify the active agnostic mixture component and by communicating only sparse-quantized model-copy innovations when the event trigger is activated. The accompanying analysis establishes stationarity of the fairness-aware objective while preserving the standard $\mathcal{O}(T^{-1/2})$ non-convex stationarity rate, and explicitly accounting for stochastic-gradient noise, stale neighbor-copy error, compression, and event-triggering. We also derive PAC-Bayes generalization guarantees for the agnostic mixture objective, thereby characterizing both optimization and statistical aspects of fairness-aware decentralized learning.
\vspace{-0.3cm}
\subsection{Related Work}
A straightforward implementation of FL with stochastic gradient descent requires frequent exchange of high-dimensional gradient or model-update vectors, which may contain millions of parameters. This creates substantial communication overhead, particularly in wireless edge networks with limited radio resources. One common approach to mitigate this bottleneck is to compress the gradient information before transmission~\cite{pmlr-v97-koloskova19a}. A simple example is sign-based compression, commonly known as SignSGD, where only the sign of each gradient coordinate is communicated~\cite{bernstein18}. These methods improve scalability and avoid a single point of failure, but most existing decentralized optimization algorithms focus on learning a single global model or minimizing an average objective. In contrast, this work studies a fully decentralized setting in which clients learn personalized models while jointly accounting for fairness and communication constraints.

In decentralized settings, Assran et al.~\cite{pmlr-v97-assran19a} and Tatarenko and Touri~\cite{tatarenko17} analyze stochastic gradient-push methods for non-convex objectives, where the main focus is on approximating distributed averaging over directed networks rather than reducing communication through gradient compression. In contrast, Tang et al.~\cite{tang18} study unbiased stochastic compression mechanisms for exchanging gradient or model information across decentralized nodes. Personalized FL addresses statistical heterogeneity by allowing clients to learn task-specific models adapted to their local data~\cite{smith2017federated,NEURIPS2020_tdinh, youchaoqun23}. These methods improve local adaptation under heterogeneous data, but they generally do not incorporate agnostic client-level fairness objectives or compressed event-triggered communication in fully decentralized networks.

Communication reduction has been widely studied through sparsification~\cite{aji2017sparse}, quantization~\cite{alistarh2017qsgd}, and error-feedback correction~\cite{karimireddy2019error}. Event-triggered communication further reduces bandwidth consumption by allowing clients to communicate only when local updates are sufficiently informative. SPARQ-SGD~\cite{singh23}, for example, combines sparsification, quantization, and event-triggered updates for decentralized optimization. Recent works have also established non-convex convergence guarantees for decentralized SGD under compression, network constraints, and related communication limitations~\cite{pmlr-v202-yan23a,xinleiyi23,koloskova2020unified}. However, these methods typically optimize an average objective or a single global model, and do not analyze the joint effect of graph-based personalization, agnostic fairness, and compressed decentralized communication.

The above directions address important aspects of distributed learning, but largely in isolation: personalization handles statistical heterogeneity, fairness-aware FL controls client-level disparities, and compressed decentralized optimization reduces communication, summarized in Table~\ref{tab:lit_review}. In contrast, DMFL-SQ provides a unified framework that jointly incorporates graph-based personalization, agnostic fairness, and sparsified, quantized, event-triggered decentralized communication, while providing both nonconvex convergence and PAC-Bayes generalization bounds for the resulting fairness-aware mixture objective.

\begin{table}[h]
\centering
\caption{Literature Review}
\label{tab:lit_review}
\renewcommand{\arraystretch}{1.15}
\setlength{\tabcolsep}{3.5pt}
\begin{tabular}{|p{2.9cm}|c|c|c|c|c|c|c|}
\hline
\textbf{Feature} & \cite{lian2017can}& \cite{pmlr-v97-koloskova19a}& \cite{smith2017federated}& \cite{mohri2019agnostic} & \cite{singh23} & \cite{li24}\cite{chen24}&\textbf{Proposed} \\
\hline
{Fully decentralized} 
& \checkmark & \checkmark &  & &  \checkmark & &  \checkmark \\
\hline
\makecell[l]{Personalized} 
&  &  & \checkmark &  &  && \checkmark \\
\hline
{Fairness-aware objective} 
&  &  &  & \checkmark &  && \checkmark \\
\hline
\makecell[l]{Compressed\\communication} 
& \checkmark & \checkmark &  &  & \checkmark & \checkmark &\checkmark \\
\hline
\makecell[l]{Event-triggered\\communication} 
&  &  &  &  & \checkmark & &\checkmark \\
\hline
\makecell[l]{Non-convex\\convergence guarantee} 
& \checkmark &  &  &  & \checkmark && \checkmark \\
\hline
\makecell[l]{Generalization\\guarantee} 
&  &  &  & \checkmark &  && \checkmark \\
\hline
\end{tabular}
\end{table}
\vspace{-0.2cm}
\subsection{Contributions and Organization}
The main contributions of this paper are summarized as follows.
\begin{itemize}
    \item \textbf{Unified decentralized fair multi-task learning framework:}
    We formulate decentralized learning under communication constraints as a unified problem involving graph-based personalization, agnostic fairness, and compressed event-triggered communication. This provides a single framework for jointly studying personalization, client-level fairness, and bandwidth efficiency in fully decentralized networks.

    \item \textbf{Communication-efficient decentralized algorithm:}
    We propose \textsc{DMFL-SQ}, a decentralized multi-task learning algorithm that combines local stochastic updates, graph-based model coupling, sparsified and quantized communication, and event-triggered synchronization. The algorithm reduces communication overhead while preserving fairness-aware collaboration among heterogeneous clients.

    \item \textbf{Non-convex convergence guarantees:}
    We establish finite-time convergence guarantees for general non-convex objectives. The analysis explicitly separates the optimization error, communication-induced neighbor-copy residual, compression residual error, and event-triggering error, and shows that \textsc{DMFL-SQ} achieves an $\mathcal O(T^{-1/2})$ rate in expected squared
Moreau-envelope stationarity, matching the standard stochastic nonconvex optimization rate while accounting for
decentralized graph coupling, fairness selection, and compressed,
event-triggered communication.

    \item \textbf{PAC-Bayes generalization analysis:}
    We derive PAC-Bayes generalization bounds for the fairness-aware
mixture objective. The analysis accounts for the complexity of the
mixture class and quantifies the statistical cost of controlling both
worst-client and general client-mixture risks.
\item \textbf{Empirical validation on controlled and real heterogeneous data:}
    We evaluate \textsc{DMFL-SQ} on CIFAR-10 under controlled non-iid partitions and on the real MUSMET EEG dataset, which provides naturally heterogeneous musician-specific data. The results demonstrate the trade-off among accuracy, fairness, personalization, and communication efficiency, showing that fairness-aware decentralized learning can improve performance equity across clients while substantially reducing communication.
\end{itemize}


Overall, this paper shows that fairness-aware decentralized learning does not necessarily require sacrificing asymptotic optimization efficiency. When graph-based personalization, agnostic fairness, and compressed decentralized communication are jointly designed, fair and personalized models can be learned under stringent communication constraints while preserving the standard non-convex convergence behavior of stochastic gradient methods.

An outline of the remainder of the paper is as follows. The system model and problem formulation are described in Section~\ref{sec:sys_model}. Section~\ref{sec:prop_algo} discusses the DMFL-SQ algorithm. 
Section~\ref{sec:pacbayes} presents the PAC-Bayes generalization analysis, while Section~\ref{sec:convergence} establishes the convergence guarantees. The simulation results and conclusions are described in Section~\ref{sec:sim_res} and Section~\ref{sec:concl}, respectively.
\vspace{-0.2cm}
\section{System Model and Problem Formulation} \label{sec:sys_model}
We consider a decentralized FL system comprising $n$ clients
connected through a connected undirected communication graph
$\mathcal G=(\mathcal V,\mathcal E)$. Each client
$i\in\mathcal V$ maintains a personalized local model
$w_i\in\mathbb R^d$ and exchanges information only with its neighbors
$j\in\mathcal N_i$, as illustrated in
Fig.~\ref{fig:system_model}. Let $A=[a_{ij}]$ denote the symmetric
nonnegative graph-weight matrix, where $a_{ij}>0$ if
$(i,j)\in\mathcal E$ and $a_{ij}=0$ otherwise. These weights define
the graph-coupling regularizer that encourages neighboring personalized
models to remain similar while preserving client-specific solutions.

\begin{figure}[t]
\centering
\begin{tikzpicture}[
    node distance=1.5cm,
    client/.style={
        circle,
        draw=blue!70!black,
        fill=blue!10,
        thick,
        minimum size=1.25cm,
        align=center,
        font=\small
    },
    data/.style={
        rectangle,
        rounded corners,
        draw=gray!70!black,
        fill=gray!10,
        thick,
        minimum width=1.2cm,
        minimum height=0.55cm,
        align=center,
        font=\scriptsize
    },
    model/.style={
        rectangle,
        rounded corners,
        draw=green!50!black,
        fill=green!10,
        thick,
        minimum width=1.25cm,
        minimum height=0.55cm,
        align=center,
        font=\scriptsize
    },
    fairbox/.style={
        rectangle,
        rounded corners,
        draw=red!70!black,
        fill=red!8,
        thick,
        minimum width=4.7cm,
        minimum height=1.25cm,
        align=center,
        font=\scriptsize
    },
    commbox/.style={
        rectangle,
        rounded corners,
        draw=purple!70!black,
        fill=purple!8,
        thick,
        minimum width=5.0cm,
        minimum height=1.15cm,
        align=center,
        font=\scriptsize
    },
    edge/.style={->, thick, dashed, purple!75!black},
    gedge/.style={<->, thick, dashed, blue!70!black}
]

\node[client] (c1) at (-3.0,0) {Node\\$1$};
\node[client] (c2) at (0,2.4) {Node\\$2$};
\node[client] (c3) at (3.0,0) {Node\\$3$};

\node[model] (m1) at (-4.0,-1.15) {$w_1,\;F_1(w_1)$};
\node[data]  (d1) at (-1.9,-1.15) {$\mathcal D_1\sim P_1$};

\node[data]  (d2) at (-1.25,3.35) {$\mathcal D_2\sim P_2$};
\node[model] (m2) at (1.25,3.35) {$w_2,\;F_2(w_2)$};

\node[data]  (d3) at (2.1,-1.15) {$\mathcal D_3\sim P_3$};
\node[model] (m3) at (3.7,-1.35) {$w_3,\;F_3(w_3)$};

\draw[->, gray!70!black, thick] (d1) -- (c1);
\draw[->, gray!70!black, thick] (m1) -- (c1);

\draw[->, gray!70!black, thick] (d2) -- (c2);
\draw[->, gray!70!black, thick] (m2) -- (c2);

\draw[->, gray!70!black, thick] (d3) -- (c3);
\draw[->, gray!70!black, thick] (m3) -- (c3);

\draw[gedge] (c1) -- (c2);
\draw[gedge] (c2) -- (c3);
\draw[gedge] (c1) -- (c3);

\node[font=\scriptsize, blue!70!black, align=center] at (0,-0.50)
{Decentralized graph coupling via \(a_{ij}\)\\[-0.2mm]
\(\mathcal R_{\mathcal G}(\mathbf w)
=\frac{\alpha}{2}\sum_{(i,j)\in\mathcal E}a_{ij}\|w_i-w_j\|^2\)};

\node[fairbox] (fair) at (0,5.05)
{
Agnostic fairness envelope\\[-0.8mm]
$\displaystyle \Psi(\mathbf w)=
\sup_{\lambda\in\Lambda}\sum_{i=1}^{n}\lambda_iF_i(w_i)$\\[-0.5mm]
Scalar max-consensus on \((\widehat f_i^t,i)\)
};

\draw[->, dashed, red!70!black, thick]
(m1.north) to[out=115, in=-165]
node[left, font=\scriptsize] {$\widehat f_1^t$} (fair.west);

\draw[->, dashed, red!70!black, thick]
(m2.north) --
node[right, font=\scriptsize] {$\widehat f_2^t$} (fair.south);

\draw[->, dashed, red!70!black, thick]
(m3.north) to[out=80, in=-40]
node[right, font=\scriptsize] {$\widehat f_3^t$} (fair.east);

\node[commbox] (comm) at (0,-2.25)
{
Sparse-quantized event-triggered model-copy update\\[0.5mm]
$\displaystyle s_i^{t+1}=w_i^{t+1}-\widehat w_i^t,\quad
c_i^t=\mathcal C(s_i^{t+1})$\\[-0.2mm]
Transmit if \(\|s_i^{t+1}\|\ge\vartheta_t\);
neighbors update their reconstructed copies using
\(c_i^t\)
};

\draw[edge] (c1.south) -- (comm.north west);
\draw[edge] (c2.south) -- (comm.north);
\draw[edge] (c3.south) -- (comm.north east);

\node[font=\scriptsize, align=center, green!40!black] at (0,1.15)
{Local personalized update at every round:\\[-0.2mm]
\(w_i^{t+1}=w_i^t-\gamma_t h_i^t\)};

\end{tikzpicture}
\caption{Illustration of DMFL-SQ. Each client maintains a personalized model
\(w_i\), performs local learning at every round, and communicates only
sparse-quantized event-triggered model-copy innovations. Neighboring models
are softly coupled through \(\mathcal R_{\mathcal G}(\mathbf w)\), while
the active fairness component is selected through scalar max-consensus over
local loss estimates.}
\label{fig:system_model}
\end{figure}
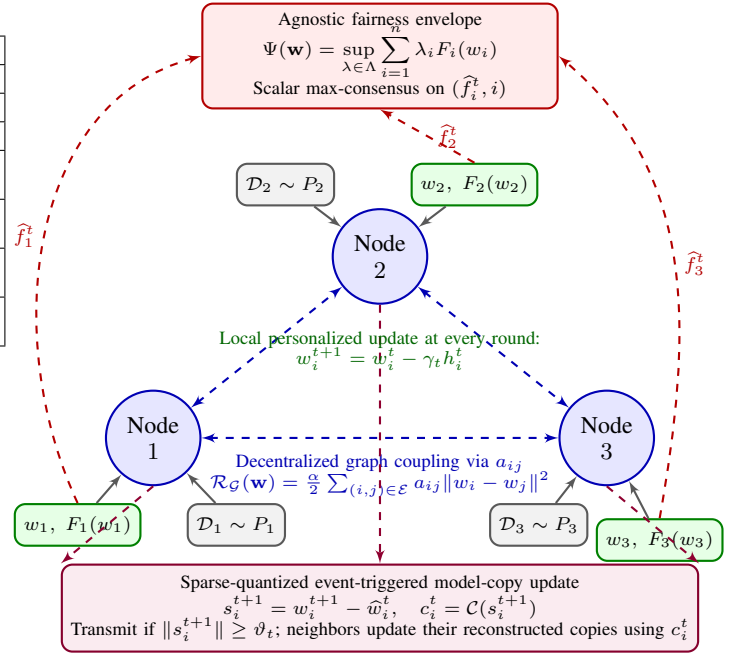

Each client $i$ has access to a local dataset $\mathcal{D}_i = \{(x_{i,k}, y_{i,k})\}_{k=1}^{m_i}$ generated from an underlying distribution $\mathbb{P}_i$, where the client distributions $\{\mathbb{P}_i\}_{i=1}^n$ may differ across clients, capturing statistical heterogeneity (non-iid data). The expected local loss is defined as
\begin{equation}
F_i(w_i) = \mathbb{E}_{(x,y)\sim \mathbb{P}_i}[\ell_i(w_i; x,y)],
\end{equation}
where $\ell_i(\cdot)$ is a potentially non-convex loss function that is assumed to be $L$-smooth with respect to the model parameter.

Unlike standard FL, which enforces a single shared global model, we adopt a \emph{multi-task} formulation in which each client learns a personalized model $w_i$ while collaborating with neighboring nodes. Let $\mathbf{w} = [w_1, \dots, w_n]$ denote the collection of client-specific models. Collaboration among neighboring clients is modeled through the graph-Laplacian regularizer
\begin{equation}
\mathcal{R}_{\mathcal{G}}(\mathbf{w}) = 
\frac{\alpha}{2}\sum_{(i,j)\in\mathcal{E}} a_{ij} \|w_i - w_j\|^2,
\end{equation}
where $\alpha>0$ controls the strength of graph-based inter-client coupling and $a_{ij}\ge 0$ are graph-regularization weights.

To incorporate \emph{fairness}, we adopt an \emph{agnostic mixture-based formulation}~\cite{mohri2019agnostic}. 
We consider a mixture distribution $P_{\lambda} = \sum_{i=1}^n \lambda_i \mathbb{P}_i$, where $\lambda \in \Lambda \subseteq \Delta_n$ and $\Delta_n$ denotes the probability simplex. We define the agnostic fairness envelope as
\vspace{-0.5cm}
\begin{align}
\Psi(\mathbf w) :=\sup_{\lambda\in\Lambda}
\sum_{i=1}^{n}\lambda_i F_i(w_i).
\end{align}
The proposed decentralized fair multi-task objective is then
\begin{equation}
\min_{\{w_i\}_{i=1}^n}
\;
\mathcal L(\mathbf w)
\;:=\;
\sum_{i=1}^{n} F_i(w_i)
\;+\;
\mathcal{R}_{\mathcal{G}}(\mathbf{w})
\;+\;
\rho\,\Psi(\mathbf w),
\label{eq:main-objective}
\end{equation}
where $\rho \ge 0$ controls the trade-off between aggregate performance and worst-case (fairness) risk across clients. The envelope $\Psi(\mathbf w)$ corresponds to the worst-case mixture loss over clients, as studied in agnostic federated learning ~\cite{mohri2019agnostic}. The formulation in \eqref{eq:main-objective} can be viewed as a penalized version of the pure min-max problem, allowing a continuous trade-off between average accuracy and fairness. Setting $\rho=0$ recovers standard decentralized multi-task learning, while larger values of $\rho$ increasingly emphasize worst-case client performance. Unlike classical FL objectives that learn a single global model, the formulation in~\eqref{eq:main-objective} yields personalized models $w_i$ that adapt to heterogeneous data while maintaining fairness across the network.

We study this problem under general \emph{non-convex} local objectives $F_i(\cdot)$ and develop a decentralized stochastic algorithm that achieves communication efficiency through sparsified, quantized, and event-triggered updates. The proposed algorithm, \textsc{DMFL-SQ}, admits convergence guarantees to first-order stationary points under general non-convex objectives.
\vspace{-0.5cm}
\subsection{Assumptions}
\label{sec:assum}
We make the following standard assumptions commonly used in decentralized stochastic
optimization and federated learning analyses.
\begin{assmon}[Smoothness]
\label{ass:smoothness}
Each local objective $F_i:\mathbb{R}^d\to\mathbb{R}$ is $L$-smooth, i.e., for all
$u,v\in\mathbb{R}^d$,
\begin{equation}
\|\nabla F_i(u)-\nabla F_i(v)\|\le L\|u-v\|.
\end{equation}
\end{assmon}

\begin{assmon}[Independent Loss Selection and Unbiased
Stochastic Gradients]
\label{ass:variance}
Let $\mathcal F_t$ denote the algorithmic history available before
the mini-batches at round $t$ are sampled. At every round, the
collection of gradient mini-batches
$\{B_{i,\mathrm{grad}}^t\}_{i=1}^n$ is conditionally independent
of the collection of selection mini-batches
$\{B_{i,\mathrm{sel}}^t\}_{i=1}^n$ given $\mathcal F_t$.

The stochastic gradient computed from
$B_{i,\mathrm{grad}}^t$ satisfies
\begin{equation}
\mathbb E
\left[
g_i^t
\,\middle|\,
\mathcal F_t,\widehat{\lambda}^{\,t}
\right]
=
\nabla F_i(w_i^t),
\end{equation}
and
\begin{equation}
\mathbb E
\left[
\left\|
g_i^t-\nabla F_i(w_i^t)
\right\|^2
\,\middle|\,
\mathcal F_t,\widehat{\lambda}^{\,t}
\right]
\leq
\sigma_i^2.
\end{equation}
Let $ \bar{\sigma}^2
:= \frac{1}{n}\sum_{i=1}^n\sigma_i^2 $ denote the average stochastic-gradient variance.
\end{assmon}

\begin{assmon}[Communication Graph and Coupling Weights] 
\label{ass:graph} 
The communication graph $\mathcal G=(\mathcal V,\mathcal E)$ is undirected and connected. The graph-coupling weights satisfy $a_{ij}=a_{ji}\ge 0$, with $a_{ij}>0$ only if $(i,j)\in\mathcal E$, and $a_{ij}=0$ otherwise. Moreover, the weighted degree is uniformly bounded, i.e., \[ d_{\mathcal G}:=\max_{i\in\mathcal V}\sum_{j\in\mathcal N_i}a_{ij}<\infty . \] Let \(L_{\mathcal G}\) denote the weighted graph Laplacian associated with \(\{a_{ij}\}\). Then the graph regularizer \[ \mathcal R_{\mathcal G}(\mathbf w) = \frac{\alpha}{2} \sum_{(i,j)\in\mathcal E} a_{ij}\|w_i-w_j\|^2 \] is smooth with constant $ L_R=\alpha\lambda_{\max}(L_{\mathcal G})$.
\end{assmon}

\begin{assmon}[Contractive Compressor]
\label{ass:compression}
The possibly randomized compression operator $\mathcal{C}:\mathbb{R}^d\to\mathbb{R}^d$ satisfies
the contraction property: there exists $\omega\in(0,1]$ such that 
\begin{equation}
\mathbb{E}\|\mathcal{C}(v)-v\|^2 \le (1-\omega)\|v\|^2, \quad \forall v \in \mathbb{R}^d,
\end{equation}
where the expectation is taken over the randomness of the compressor.
\end{assmon}
Assumptions~\ref{ass:smoothness}-\ref{ass:compression}
control the regularity and stochastic moments of the local learning
directions, the graph-based personalized coupling, and compressed
communication. Assumption~\ref{ass:graph} ensures that the
graph-coupling direction is well defined and that errors arising from
stale reconstructed neighbor models can be controlled through the
communication residual. Unlike consensus-based decentralized SGD,
DMFL-SQ does not impose model consensus through a mixing step;
personalization is instead maintained through the graph regularizer
$\mathcal R_{\mathcal G}$.
\vspace{-0.2cm}
\section{DMFL-SQ Algorithm}
\label{sec:prop_algo}
We now present the proposed decentralized multi-task fair FL with sparsified and quantized communication, referred to as \textsc{DMFL-SQ}, also summarized in Algorithm~\ref{alg:dmflsq}. The method enables each client to perform local stochastic updates, exchange compressed information with its neighbors only when necessary, and promote fairness across heterogeneous clients via the agnostic mixture-based objective in~\eqref{eq:main-objective}.

At iteration $t$, each client $i$ draws two conditionally
independent mini-batches, $B_{i,\mathrm{sel}}^t$ and
$B_{i,\mathrm{grad}}^t$. The first mini-batch is used to estimate
the local loss for fairness-component selection:
\begin{equation}
\widehat f_i^t
=
\frac{1}{|B_{i,\mathrm{sel}}^t|}
\sum_{z\in B_{i,\mathrm{sel}}^t}
\ell_i(w_i^t;z).
\label{eq:loss_estimate}
\end{equation}
The second mini-batch is used to compute the stochastic gradient
\begin{equation}
g_i^t
=
\frac{1}{|B_{i,\mathrm{grad}}^t|}
\sum_{z\in B_{i,\mathrm{grad}}^t}
\nabla \ell_i(w_i^t;z).
\label{eq:stochastic_gradient}
\end{equation}
Let $\mathcal F_t$ denote the algorithmic history available before
the selection and gradient mini-batches at round $t$ are sampled.
Since $B_{i,\mathrm{grad}}^t$ is conditionally independent of
$\{B_{k,\mathrm{sel}}^t\}_{k=1}^n$ given $\mathcal F_t$, we have
\begin{equation}
\mathbb{E}\!\left[
g_i^t
\,\middle|\,
\mathcal{F}_t,\widehat{\mathbf f}^{\,t},
\widehat{\lambda}^{t}
\right]
=
\nabla F_i(w_i^t).
\label{eq:conditional-unbiasedness}
\end{equation}
For any mixture vector $\lambda\in\Lambda$, define the population
and empirical mixture losses, respectively, as
\begin{equation}
\Phi_{\lambda}(\mathbf w)
:=
\sum_{i=1}^n \lambda_i F_i(w_i),
\qquad
\widehat{\Phi}_{\lambda}^{\,t}
:=
\sum_{i=1}^n \lambda_i \widehat f_i^t.
\label{eq:population_empirical_mixture}
\end{equation}
For a general mixture class $\Lambda$, the empirical fairness
mixture is selected according to
$\widehat{\lambda}^{\,t}
\in
\arg\max_{\lambda\in\Lambda}
\widehat{\Phi}_{\lambda}^{\,t}$. For the full-simplex case $\Lambda=\Delta_n$, this reduces to
$i_t^\star
\in
\arg\max_{i\in[n]}
\widehat f_i^t,
\widehat{\lambda}_i^t
=
\mathbbm{1}\{i=i_t^\star\}$. The stochastic fairness direction implemented by client $i$ is
defined as
$\widehat{\xi}_i^t
=
\widehat{\lambda}_i^t g_i^t$. We use the notation $\widehat{\xi}_i^t$ to emphasize that this
direction is constructed from mini-batch loss estimates and is not
necessarily a population subgradient of the fairness envelope
$\Psi$.

To quantify the discrepancy between empirical fairness selection
and the population fairness envelope, define
the population fairness-selection gap
\begin{equation}
\varepsilon_{\Psi,t}
:=
\Psi(\mathbf{w}^t)
-
\Phi_{\widehat{\lambda}^{\,t}}(\mathbf{w}^t)
\geq 0.
\label{eq:fairness_selection_gap}
\end{equation}
For $\Lambda=\Delta_n$, this becomes
$
\varepsilon_{\Psi,t}
=
\max_{i\in[n]} F_i(w_i^t)
-
F_{i_t^\star}(w_{i_t^\star}^t)$. Thus, $\varepsilon_{\Psi,t}=0$ when the empirically selected client
is also a population-active client. In general, exact scalar
max-consensus guarantees agreement on the empirical maximizer but
does not imply that the selected client maximizes the population
loss.

Let $
\delta_t
:=
\sup_{\lambda\in\Lambda}
\left|
\widehat{\Phi}_{\lambda}^t
-
\Phi_{\lambda}(\mathbf w^t)
\right|$. Because $\Lambda\subseteq\Delta_n$, $\delta_t
\leq
\max_{i\in[n]}
\left|
\widehat f_i^t-F_i(w_i^t)
\right|$. Since $\widehat{\lambda}^{\,t}$ maximizes
$\widehat{\Phi}_{\lambda}^t$ over $\Lambda$, the population
fairness-selection gap satisfies
\begin{equation}
\varepsilon_{\Psi,t}
\leq
2\delta_t.
\label{eq:selection_gap_bound}
\end{equation}
Thus, the empirically selected fairness mixture becomes increasingly
accurate in terms of the population fairness objective as the uniform
loss-estimation error decreases.

\begin{algorithm}[]
\caption{DMFL-SQ: Decentralized Multi-Task Fair Learning with
Sparse-Quantized Event-Triggered Communication}
\label{alg:dmflsq}
\begin{algorithmic}[1]
\State \textbf{Input:} stepsizes $\{\gamma_t\}$, trigger thresholds
$\{\vartheta_t\}$, graph weights $\{a_{ij}\}$, compression operator
$\mathcal C$, fairness weight $\rho$, graph regularization weight
$\alpha$, fairness-consensus rounds $K_\Psi$.
\State Initialize $w_i^0$ for all clients $i\in\{1,\ldots,n\}$.
\State Initialize the locally stored transmitted copies
$\widehat{w}_i^0=w_i^0$ $\forall$ $i$, the neighbor copies
$\widehat{w}_{i\to j}^0=w_i^0$ $\forall$ $(i,j)\in\mathcal{E}$.
\For{$t=0,\ldots,T-1$}

    \For{each client $i=1,\ldots,n$ in parallel}
\State Sample conditionally independent mini-batches
$B_{i,\mathrm{sel}}^t$ and $B_{i,\mathrm{grad}}^t$ and compute
\[
\widehat f_i^t
=
\frac{1}{|B_{i,\mathrm{sel}}^t|}
\sum_{z\in B_{i,\mathrm{sel}}^t}
\ell_i(w_i^t;z),
\]
\[
g_i^t
=
\frac{1}{|B_{i,\mathrm{grad}}^t|}
\sum_{z\in B_{i,\mathrm{grad}}^t}
\nabla\ell_i(w_i^t;z).
\]
 \EndFor
\State Run $K_\Psi$ scalar max-consensus rounds on
$(\widehat f_i^t,i)$ to obtain $
i_t^\star
\in
\arg\max_{i\in[n]}
\widehat f_i^t.
$

\State Set $
\widehat{\lambda}_i^t
=
\mathbbm{1}\{i=i_t^\star\}$,
$
\widehat{\xi}_i^t
=
\widehat{\lambda}_i^t g_i^t$.
    \For{each client $i=1,\ldots,n$ in parallel}
        \State Compute the graph-coupling direction
        \[
        r_i^t
        =
        \alpha
        \sum_{j\in\mathcal N_i}
        a_{ij}
        \left(
        w_i^t-\widehat w_{j\to i}^t
        \right).
        \]

\State Form the local stochastic direction
\[
h_i^t
=
g_i^t+r_i^t+\rho\widehat{\xi}_i^t
=
\left(1+\rho\widehat{\lambda}_i^t\right)g_i^t+r_i^t.
\]
\State Update the personalized model
        \[
        w_i^{t+1}=w_i^t-\gamma_t h_i^t.
        \]
        \State Compute the model-copy innovation
        \[
        s_i^{t+1}=w_i^{t+1}-\widehat w_i^t,
        \]
        where $\widehat w_i^t$ denotes the most recent transmitted copy
        of $w_i$ reconstructed by its neighbors.
        \If{$\|s_i^{t+1}\|\ge \vartheta_t$}
            \State Compress and transmit
            $c_i^t=\mathcal C(s_i^{t+1})
            $         to all neighbors $j\in\mathcal N_i$.
            \State Update the locally stored transmitted copy
            \[
            \widehat w_i^{t+1}=\widehat w_i^t+c_i^t.
            \]
        \Else
            \State Set $c_i^t=0$ and
            \[
            \widehat w_i^{t+1}=\widehat w_i^t.
            \]
        \EndIf
    \EndFor

    \For{each client $i=1,\ldots,n$ in parallel}
        \State For every received $c_j^t$ from $j\in\mathcal N_i$, update
        \[
        \widehat w_{j\to i}^{t+1}
        =
        \widehat w_{j\to i}^{t}
        +
        c_j^t.
        \]
        \State If no message is received from neighbor $j$, set
        \[
        \widehat w_{j\to i}^{t+1}
        =
        \widehat w_{j\to i}^{t}.
        \]
    \EndFor

\EndFor
\State \textbf{Output:} personalized models $\{w_i^T\}_{i=1}^n$.
\end{algorithmic}
\end{algorithm}

\noindent\textbf{
Fairness-consensus implementation:}
For $\Lambda=\Delta_n$, DMFL-SQ runs $K_{\Psi}$ rounds of
scalar max-consensus over the pairs $(\widehat f_i^t,i)$, using a
fixed lexicographic tie-breaking rule. If
$K_{\Psi}\geq\operatorname{diam}(\mathcal G)$, all clients recover
the same empirical maximizer
$i_t^\star
\in
\arg\max_{i\in[n]}
\widehat f_i^t$.
Exact max-consensus therefore eliminates disagreement among clients
regarding the empirical fairness component. It does not, by itself,
guarantee that $i_t^\star$ is an active maximizer of the population
fairness envelope. The resulting statistical discrepancy is
quantified by $\varepsilon_{\Psi,t}$ in
\eqref{eq:fairness_selection_gap} and is explicitly retained in the
convergence analysis.

Each client maintains reconstructed copies of its neighbors' most recently
communicated models. Let $\widehat w_{j\to i}^t$ denote the latest copy of client $j$'s model available at client $i$. These copies are used to evaluate the graph-coupling term in the personalized objective.
Specifically, client $i$ computes
\begin{equation}
r_i^t
=
\alpha
\sum_{j\in\mathcal N_i}
a_{ij}
\left(
w_i^t-\widehat w_{j\to i}^t
\right),
\end{equation}
which is a communication-efficient approximation of
$\nabla_{w_i}\mathcal R_{\mathcal G}(\mathbf w^t)$.
Although the graph regularizer
$\mathcal R_{\mathcal G}(\mathbf w)$ is defined using the current
neighboring models, a decentralized event-triggered implementation cannot
access $w_j^t$ at every round. Therefore, client $i$ evaluates the
graph-coupling direction using the latest reconstructed neighbor copy
$\widehat w_{j\to i}^t$. The resulting discrepancy
$w_j^t-\widehat w_{j\to i}^t$ is explicitly captured by the
communication residual $\mathcal E^t$ and is accounted for in the
convergence analysis. The local stochastic direction implemented by client $i$ is
\begin{equation}
h_i^t
=
g_i^t+r_i^t+\rho\widehat{\xi}_i^t
=
\left(1+\rho\widehat{\lambda}_i^t\right)g_i^t+r_i^t.
\label{eq:implemented_direction}
\end{equation}
Client $i$ then updates its personalized model according to
\begin{equation}
w_i^{t+1}
=
w_i^t-\gamma_t h_i^t.
\label{eq:model_update_revised}
\end{equation}
After the local model update, client $i$ checks whether its current
model has changed sufficiently relative to the last version reconstructed
by its neighbors. Let $\widehat w_i^t$ denote this last transmitted
copy and define the model-copy innovation
as $s_i^{t+1}= w_i^{t+1}-\widehat w_i^t$.

Client $i$ communicates only if
$\|s_i^{t+1}\|\ge \vartheta_t$, where $\vartheta_t$ is a time-varying trigger threshold. If the trigger
condition is satisfied, client $i$ sends the sparse-quantized update
$c_i^t = \mathcal C(s_i^{t+1})$
to all neighbors and updates its locally stored transmitted copy as
$\widehat w_i^{t+1} = \widehat w_i^t+c_i^t$.
If the trigger condition is not satisfied, no message is sent, and $c_i^t=0$, $\widehat w_i^{t+1} = \widehat w_i^t$.

Upon receiving $c_j^t$ from a neighbor $j$, client $i$ updates its
local reconstructed copy as
$\widehat w_{j\to i}^{t+1} =\widehat w_{j\to i}^{t} + c_j^t$. If no message is received from neighbor $j$, the copy is kept unchanged:
$\widehat w_{j\to i}^{t+1} =\widehat w_{j\to i}^{t}$.
The event trigger in Algorithm~1 controls only
communication and does not suppress local learning. Every client
updates its personalized model at each iteration. When the trigger is
inactive, the client skips broadcasting its model-copy innovation,
while retaining the updated local model. This mechanism reduces communication while preserving personalization. Unlike consensus-based decentralized SGD, DMFL-SQ does not average the client models after every local update. Instead, neighboring models are softly coupled through the graph regularizer $\mathcal R_{\mathcal G}(\mathbf w)$, while sparse, quantized, and
event-triggered communication maintains sufficiently accurate neighbor model copies. Therefore, client-specific variation is controlled by the personalized objective itself rather than being suppressed by an explicit gossip-averaging step.

\noindent\textbf{Communication accounting.}
Let $P$ denote the number of model parameters,
$k_i^t$ the number of nonzero coordinates transmitted by client $i$
at round $t$, $b_q$ the quantization precision, and
$\deg_i=|\mathcal N_i|$ its graph degree. Define the trigger indicator
$\chi_i^t=\mathbf 1\{\|s_i^{t+1}\|\geq\vartheta_t\}$. The model-copy
communication cost at round $t$ is
\begin{align}
B_{\mathrm{model},t}
&=
\sum_{i=1}^n \deg_i
\bigg[
b_{\mathrm{trig}}
+
\chi_i^t
\bigg(
k_i^t b_q
+
\min\!\left\{
k_i^t\lceil\log_2P\rceil,P
\right\}
\nonumber\\
&+
b_{\mathrm{scale}}
+
b_{\mathrm{hdr}}
\bigg)
\bigg],
\label{eq:model_comm_cost}
\end{align}
where the terms account for the quantized values, coordinate indices
or a binary mask, the quantizer scale, message header, and trigger
metadata, respectively. If silence indicates an inactive trigger, we
set $b_{\mathrm{trig}}=0$. The factor $\deg_i$ accounts for delivery
to all neighbors; for an undirected graph,
$\sum_i\deg_i=2|\mathcal E|$.

The scalar fairness-consensus step incurs
\begin{equation}
B_{\Psi,t}
=
2K_\Psi|\mathcal E|
\left(
b_{\mathrm{loss}}
+
\lceil\log_2 n\rceil
\right),
\label{eq:fairness_comm_cost}
\end{equation}
since the pair $(\widehat f_i^t,i)$ is exchanged in both directions
over every edge during each of the $K_\Psi$ max-consensus rounds.
Hence, the total communication over $T$ rounds and the average
communication per client are
\begin{equation}
B_{\mathrm{tot}}
=
\sum_{t=0}^{T-1}
\left(
B_{\mathrm{model},t}+B_{\Psi,t}
\right),
\qquad
\overline B_{\mathrm{client}}
=
\frac{B_{\mathrm{tot}}}{n}.
\label{eq:total_comm_cost}
\end{equation}

Overall, DMFL-SQ updates each personalized model using a stochastic approximation of the full objective, while communication is used only to maintain reconstructed neighbor copies for evaluating the graph
regularizer. Unlike consensus-based decentralized SGD, the client models
are not averaged after every local update. Instead, personalization is
preserved because inter-client coupling is controlled softly through
$\mathcal R_{\mathcal G}(\mathbf w)$, and the sparse-quantized
event-triggered mechanism reduces communication by updating
$\widehat w_{j\to i}^t$ only when necessary.
\vspace{-0.2cm}
\section{PAC-Bayes Generalization for DMFL-SQ}
\label{sec:pacbayes}
We establish an algorithm-agnostic PAC-Bayes generalization guarantee
for the fairness-aware agnostic mixture objective under possibly
non-convex client losses. We subsequently complement this result with a
corollary that quantifies the one-step empirical mixture-risk
perturbation induced by stale and compressed neighbor-model
information. 
Consider $n$ clients with mutually independent datasets, where
client $i$ has access to
$\mathcal D_i=\{z_{i,k}\}_{k=1}^{m_i}$, whose samples are drawn
iid from $\mathbb P_i$, with $\mathbb{P}_i \neq \mathbb{P}_j$ in general. Let $\ell_i(w_i; z) \in [0,1]$ denote the loss at client $i$.
Define the population and empirical risks as
\begin{align}
F_i(w_i) := \mathbb{E}_{z \sim \mathbb{P}_i}[\ell_i(w_i; z)],
\widehat{F}_i(w_i) := \frac{1}{m_i} \sum_{k=1}^{m_i} \ell_i(w_i; z_{i,k}). \nonumber
\end{align}

Let $\mathbf{w} := (w_1,\dots,w_n)$ denote the collection of client models.
To capture fairness-aware objectives, we consider agnostic mixture risks
defined over a set of mixture weights $\Lambda \subseteq \Delta_n$, where
$\Delta_n$ is the probability simplex:
\begin{eqnarray}
\Psi(\mathbf{w}) := \sup_{\lambda \in \Lambda} \sum_{i=1}^n \lambda_i F_i(w_i),
\vspace{0.2cm}
\widehat{\Psi}(\mathbf{w}) := \sup_{\lambda \in \Lambda} \sum_{i=1}^n \lambda_i \widehat{F}_i(w_i). \nonumber
\label{eq:mixture-risk}
\end{eqnarray}
We emphasize that no convexity is assumed on $F_i$.
\vspace{-0.1cm}
\subsection{PAC-Bayes Generalization Bound}
For a fixed mixture vector $\lambda\in\Lambda$, define
$
m_\lambda := \left(
\sum_{i=1}^n \frac{\lambda_i^2}{m_i}
\right)^{-1}.$
For a general compact mixture class, define
$
m_{\mathrm{eff}} := \inf_{\lambda\in\Lambda}m_\lambda.$

\begin{theorem}[PAC-Bayes Generalization for Agnostic Mixtures]
\label{thm:pacbayes}
Assume that $\ell_i(w_i;z)\in[0,1]$ for all clients $i\in[n]$. Let $F_i(w_i):=\mathbb E_{z\sim \mathbb{P}_i}[\ell_i(w_i;z)]$, $\widehat F_i(w_i):=\frac{1}{m_i}\sum_{k=1}^{m_i}\ell_i(w_i;z_{i,k})$,
and define the population and empirical agnostic mixture risks $\Psi(\mathbf{w}):=\sup_{\lambda\in\Lambda}\sum_{i=1}^n\lambda_iF_i(w_i)$, $\widehat\Psi(\mathbf{w}):=\sup_{\lambda\in\Lambda}\sum_{i=1}^n\lambda_i\widehat F_i(w_i)$. Let $\Pi$ be any data-independent prior over the stacked model
collection $\mathbf w$, and let $Q$ be any posterior over the same
space. Let $\tau_{\mathrm{PB}}>0$ be any deterministic PAC-Bayes
inverse-temperature parameter chosen independently of the observed
client datasets. Then, for any $\delta\in(0,1)$, the following statements hold with
probability at least $1-\delta$ over the draw of all client datasets
$\{\mathcal D_i\}_{i=1}^n$.

\textbf{Case 1: Full simplex.}
If $\Lambda=\Delta_n$, then, for any fixed
$\tau_{\mathrm{PB}}>0$, with probability at least $1-\delta$,
the following holds simultaneously for all posteriors $Q$:
\begin{align}
\mathbb{E}_{\mathbf w\sim Q}\!\left[\Psi(\mathbf w)\right]
&\leq
\mathbb{E}_{\mathbf w\sim Q}\!\left[\widehat{\Psi}(\mathbf w)\right]
+
\frac{
\mathrm{KL}(Q\Vert\Pi)+\ln n+\ln(1/\delta)
}{
\tau_{\mathrm{PB}}
}
\nonumber\\
&+
\frac{\tau_{\mathrm{PB}}}{8m_{\min}},
\label{eq:pac_bayes_full_simplex}
\end{align}
where $m_{\min}:=\min_{i\in[n]}m_i$.

\textbf{Case 2: General compact mixture class.}
If $\Lambda\subseteq\Delta_n$ is compact, then, for any
$\epsilon>0$ and any fixed $\tau_{\mathrm{PB}}>0$, with probability
at least $1-\delta$, the following holds simultaneously for all
posteriors $Q$:
\begin{align}
\mathbb{E}_{\mathbf w\sim Q}\!\left[\Psi(\mathbf w)\right]
\leq {}&
\mathbb{E}_{\mathbf w\sim Q}\!\left[\widehat{\Psi}(\mathbf w)\right]
+\epsilon +
\frac{\tau_{\mathrm{PB}}}{8m_{\mathrm{eff}}}
\nonumber\\
&+
\frac{
\mathrm{KL}(Q\Vert\Pi)
+\ln\left|\mathcal N_\epsilon
  (\Lambda,\|\cdot\|_1)\right|
+\ln(1/\delta)
}{
\tau_{\mathrm{PB}}
}
,
\label{eq:pac_bayes_general_mixture}
\end{align}
where
$ m_{\mathrm{eff}} := \inf_{\lambda\in\Lambda}
\left(\sum_{i=1}^n{\lambda_i^2}/{m_i}\right)^{-1}$.

\end{theorem}

\begin{IEEEproof}
    See Appendix~\ref{app:pacbayes}.
\end{IEEEproof}
\noindent The first case corresponds to the agnostic mixture class used in our main fairness objective. In this setting, the supremum over $\Delta_n$ is attained at a vertex, so the additional complexity cost is only $\ln n$. The second case shows how the result extends to a general compact mixture class, where the price of uniformity over $\Lambda$ appears through the
covering number $\ln|\mathcal N_\varepsilon(\Lambda,\|\cdot\|_1)|$.

We next complement this algorithm-agnostic generalization result with a
corollary that isolates the one-step empirical mixture-risk perturbation
caused by stale, sparse, quantized, and event-triggered neighbor-model
exchange. Its proof uses the communication-residual bound established
later in Lemma~\ref{lem:lemma6}.

\begin{corollary}[One-Step Effect of Stale and Compressed Neighbor Copies]
\label{cor:decentralization_compression}
Let $\mathbf d^t
=
\mathbf r^t-\nabla\mathcal R_{\mathcal G}(\mathbf w^t)$ denote the graph-coupling approximation error, and define the
exact-neighbor shadow update $
\mathbf w_{\mathrm{ex}}^{t+1}
:=
\mathbf w^t-\gamma_t(\mathbf h^t-\mathbf d^t)$. Thus, $\mathbf w_{\mathrm{ex}}^{t+1}$ uses the same stochastic gradients
and fairness selection as DMFL-SQ, but evaluates the graph-coupling
direction using the current neighbor models. Since $
\mathbf w^{t+1}-\mathbf w_{\mathrm{ex}}^{t+1}
=
-\gamma_t\mathbf d^t$,
suppose that each empirical loss $\widehat F_i$ is
$G_{\mathrm{lip}}$-Lipschitz with respect to $w_i$. Then,
\begin{align}
&\left|
\mathbb E\!\left[\widehat\Psi(\mathbf w^{t+1})\right]
-
\mathbb E\!\left[\widehat\Psi(\mathbf w_{\mathrm{ex}}^{t+1})\right]
\right|
&\le
G_{\mathrm{lip}}\gamma_t
\sqrt{nK_{\mathcal G}\,
\mathbb E[\mathcal E^t]},
\label{eq:cor_one_step_perturbation}
\end{align}
where $E^t$ is the average weighted communication residual
defined in~\eqref{eq:communication_residual}, and Lemma~\ref{lem:lemma5} gives
$\|\mathbf d^t\|^2\leq nK_GE^t$.

Let $\tau$ be sampled uniformly from
$\{0,\ldots,T-1\}$, independently of the algorithmic randomness.
Under the finite-horizon stepsize
$\gamma_t=\gamma/\sqrt{T}$,
\begin{align}
&\left|
\mathbb E\!\left[\widehat\Psi(\mathbf w^{\tau+1})\right]
-
\mathbb E\!\left[
\widehat\Psi(\mathbf w_{\mathrm{ex}}^{\tau+1})
\right]
\right|
\le
\frac{G_{\mathrm{lip}}\gamma\sqrt{nK_{\mathcal G}}}{\sqrt{T}}
\left(
\frac1T
\sum_{t=0}^{T-1}
\mathbb E[\mathcal E^t]
\right)^{1/2}.
\label{eq:cor_average_perturbation}
\end{align}

By Lemma~\ref{lem:lemma6},
\begin{align}
\frac1T
\sum_{t=0}^{T-1}
\mathbb E[\mathcal E^t]
\le
\frac{2}{\chi_0\omega T}
\left(
\mathbb E[\mathcal E^0]
+\chi_1B_h\gamma^2
+\chi_2\vartheta_0^2(1+\log T)
\right).
\label{eq:cor_residual_average}
\end{align}
Consequently,
\begin{align}
&\left|
\mathbb E\!\left[\widehat\Psi(\mathbf w^{\tau+1})\right]
-
\mathbb E\!\left[
\widehat\Psi(\mathbf w_{\mathrm{ex}}^{\tau+1})
\right]
\right|
\nonumber\\
&\qquad\le
\frac{
G_{\mathrm{lip}}\gamma\sqrt{2nK_{\mathcal G}}
}{
\sqrt{\chi_0\omega}\,T
}
\left(
\mathbb E[\mathcal E^0]
+\chi_1B_h\gamma^2
+\chi_2\vartheta_0^2(1+\log T)
\right)^{1/2}.
\label{eq:cor_final_perturbation}
\end{align}
Hence, the one-step empirical mixture-risk perturbation induced by
stale, sparse, quantized, and event-triggered neighbor-model exchange
decays as $\mathcal O\!\left(\frac{\sqrt{\log T}}{T}\right)$.
\end{corollary}

\section{Convergence Analysis}
\label{sec:convergence}
Before proving the convergence of Algorithm~\ref{alg:dmflsq},
we first summarize the regularity properties of the agnostic
fairness envelope $\Psi(\mathbf w)$ and the resulting full
personalized objective. Although the agnostic
fairness envelope
\begin{equation}
\Psi(\mathbf w)
=
\sup_{\lambda\in\Lambda}
\Phi_\lambda(\mathbf w),
\qquad
\Phi_\lambda(\mathbf w)
=
\sum_{i=1}^n \lambda_i F_i(w_i),
\end{equation}
is generally nonsmooth, it inherits weak convexity from the local
objectives. In particular, since each $F_i$ is $L$-smooth, every
mixture objective $\Phi_\lambda$ is $L$-weakly convex with respect
to the stacked model variable $\mathbf w$. Because the pointwise
supremum of functions sharing the same weak-convexity constant is
also weakly convex, $\Psi$ is $L$-weakly convex.

Consequently, the objective function
is $\kappa$-weakly convex with
$
\kappa=(1+\rho)L,
\label{eq:weak_convexity_constant}$ because $\mathcal R_{\mathcal G}$ is convex. We therefore analyze
Algorithm~\ref{alg:dmflsq} through the Moreau envelope of $\mathcal L$. This framework permits changes in the active fairness
component between successive iterates and explicitly accommodates
the empirical fairness-selection gap
$\varepsilon_{\Psi,t}$ defined in
\eqref{eq:fairness_selection_gap}.

For $K_\Psi\geq\operatorname{diam}(\mathcal G)$, scalar
max-consensus ensures that all clients agree on the same maximizer of
the empirical mini-batch losses. However, this empirical maximizer
need not coincide with a population maximizer of $\Psi$. The resulting
discrepancy is quantified by the fairness-selection gap
$\varepsilon_{\Psi,t}$ defined above and is explicitly controlled in
the subsequent Moreau-envelope convergence analysis.

We next state the additional assumptions required for the
Moreau-envelope and communication-residual analysis.
\begin{assmon}[Bounded Second Moment of the Exact-Neighbor
Population Direction]
\label{ass:bounded_direction}
Define
\begin{equation}
V_i^t
:=
\left(1+\rho\widehat{\lambda}_i^t\right)
\nabla F_i(w_i^t)
+
\nabla_{w_i}\mathcal R_{\mathcal G}(\mathbf w^t),
\mathbf V^t:=(V_1^t,\ldots,V_n^t).
\label{eq:exact_neighbor_population_direction}
\end{equation}
There exists a constant $H>0$, independent of $t$ and $T$, such
that
$\mathbb E
\left[
\|\mathbf V^t\|^2
\,\middle|\,
\mathcal F_t
\right]
\leq H^2$.
\end{assmon}

\begin{assmon}[Fairness-Selection Accuracy]
\label{ass:selection_accuracy}
Let $
\varepsilon_{\Psi,t}
= \Psi(\mathbf w^t)
-\Phi_{\widehat{\lambda}^{\,t}}(\mathbf w^t)
$ denote the population fairness-selection gap defined in
\eqref{eq:fairness_selection_gap}. There exists a deterministic
nonnegative sequence $\{\epsilon_{\Psi,t}\}_{t\geq 0}$ such that
\begin{equation}
\mathbb E
\left[
\varepsilon_{\Psi,t}
\,\middle|\,
\mathcal F_t
\right]
\leq
\epsilon_{\Psi,t}.
\label{eq:conditional_selection_accuracy}
\end{equation}
Moreover, there exists a constant $B_\Psi>0$, independent
of $T$, such that, for every training horizon $T\geq 1$,
\begin{equation}
\frac{1}{T}\sum_{t=0}^{T-1}\epsilon_{\Psi,t}
\leq \frac{B_\Psi}{\sqrt{T}}.
\label{eq:average_selection_accuracy}
\end{equation}
\end{assmon}
\noindent By~\eqref{eq:selection_gap_bound}, Assumption~\ref{ass:selection_accuracy} admits an explicit sufficient condition. Suppose that, conditional on $\mathcal F_t$, the selection-loss samples
are independent and their centered losses are
$\sigma_{\rm sel}$-sub-Gaussian, and let
$b_t:=\min_i|\mathcal B_{i,\mathrm{sel}}^t|$. Standard maximal
concentration gives
$\mathbb E[\varepsilon_{\Psi,t}\mid\mathcal F_t]
\leq 2\sigma_{\rm sel}\sqrt{2\log(2n)/b_t}$.
Thus, $b_t\geq c(t+1)\log(2n)$ implies
\begin{equation}
T^{-1}\sum_{t=0}^{T-1}\mathbb E[\varepsilon_{\Psi,t}\mid\mathcal F_t]
\leq 4\frac{\sigma_{\rm sel}}{\sqrt{T}}\sqrt{\frac{2}{c}},
\end{equation}
verifying Assumption~\ref{ass:selection_accuracy}.
For losses whose range has width $R$, one may take
$\sigma_{\rm sel}\le R/2$.
\begin{assmon}[Stepsize and Trigger Schedule]
\label{ass:trigger}
For a training horizon \(T\), DMFL-SQ uses
\[
\gamma_t=\frac{\gamma}{\sqrt T},
\qquad
\vartheta_t=\frac{\vartheta_0}{\sqrt{t+1}},
\qquad t=0,\ldots,T-1,
\]
where \(\gamma>0\) and \(\vartheta_0>0\) are constants. Hence,
\[
\sum_{t=0}^{T-1}\gamma_t\vartheta_t^2
=
\frac{\gamma\vartheta_0^2}{\sqrt T}
\sum_{t=0}^{T-1}\frac{1}{t+1}
=
\mathcal O\left(\frac{\log T}{\sqrt T}\right).
\]
After normalization by
\(\sum_{t=0}^{T-1}\gamma_t=\gamma\sqrt T\), this contributes
\[
\mathcal O\left(\frac{\log T}{T}\right)
\]
to the final stationarity bound.
\end{assmon}
\vspace{-0.2cm}
\subsection{Weak Convexity and Moreau-Envelope Regularity}
\label{sub:aux_lemma}
Recall that
$ \Psi(\mathbf w)
=
\sup_{\lambda\in\Lambda}\Phi_\lambda(\mathbf w),
\Phi_\lambda(\mathbf w)
=
\sum_{i=1}^n\lambda_iF_i(w_i)$,
where $\Lambda\subseteq\Delta_n$ is nonempty and compact.
\begin{lemma}[Weak Convexity of the Fairness Envelope and Objective Function]
\label{lem:weak_convexity}
Suppose that each $F_i$ is $L$-smooth. Then, for every
$\lambda\in\Lambda$, the mixture objective $\Phi_\lambda$ is
$L$-weakly convex. Consequently, the fairness envelope
$\Psi$ is $L$-weakly convex, and the objective function
$\mathcal L(\mathbf w)
=
\sum_{i=1}^n F_i(w_i)
+
\mathcal R_{\mathcal G}(\mathbf w)
+
\rho\Psi(\mathbf w)
$ is $\kappa$-weakly convex with
$\kappa=(1+\rho)L$.
\end{lemma}
\begin{proof}
Since each $F_i$ is $L$-smooth, every mixture objective
$\Phi_\lambda$ is $L$-weakly convex. Hence,
$\Phi_\lambda+(L/2)\|\cdot\|^2$ is convex. Taking the pointwise
supremum over $\lambda\in\Lambda$ shows that
$\Psi+(L/2)\|\cdot\|^2$ is convex, and therefore $\Psi$ is
$L$-weakly convex. Moreover, $\sum_{i=1}^n F_i(w_i)$ is
$L$-weakly convex, while $\mathcal R_{\mathcal G}$ is convex.
Thus, $\mathcal L$ is $(1+\rho)L$-weakly convex.
\end{proof}

\begin{lemma}[Approximate Fairness-Direction Inequality]
\label{lem:approx_fairness_direction}
Let $\widehat{\lambda}^{\,t}$ be the empirical fairness mixture
selected by Algorithm~\ref{alg:dmflsq}, and define the corresponding
population-gradient direction by
\begin{equation}
\bar{\xi}_i^t
:=
\widehat{\lambda}_i^t\nabla F_i(w_i^t),
\qquad
\bar{\boldsymbol{\xi}}^t
:=
(\bar{\xi}_1^t,\ldots,\bar{\xi}_n^t).
\label{eq:population_selected_fairness_direction}
\end{equation}
Then, for every $\mathbf u\in\mathbb R^{nd}$,
\begin{equation}
\Psi(\mathbf u)
\geq
\Psi(\mathbf w^t)
+
\left\langle
\bar{\boldsymbol{\xi}}^t,
\mathbf u-\mathbf w^t
\right\rangle
-
\frac{L}{2}
\|\mathbf u-\mathbf w^t\|^2
-
\varepsilon_{\Psi,t},
\label{eq:approx_fairness_weak_subgradient}
\end{equation}
where $\varepsilon_{\Psi,t}$ is defined in
\eqref{eq:fairness_selection_gap}.
Moreover, under Assumption~\ref{ass:variance},
\begin{equation}
\mathbb E
\left[
\widehat{\boldsymbol{\xi}}^t
\,\middle|\,
\mathcal F_t,\widehat{\lambda}^{\,t}
\right]
=
\bar{\boldsymbol{\xi}}^t.
\label{eq:implemented_fairness_unbiased}
\end{equation}
\end{lemma}

\begin{proof}
By the definition of $\Psi$ and
Lemma~\ref{lem:weak_convexity},
\begin{align}
\Psi(\mathbf u)
&\geq
\Phi_{\widehat{\lambda}^{\,t}}(\mathbf u) \nonumber\\
&\geq
\Phi_{\widehat{\lambda}^{\,t}}(\mathbf w^t)
+
\left\langle
\nabla\Phi_{\widehat{\lambda}^{\,t}}(\mathbf w^t),
\mathbf u-\mathbf w^t
\right\rangle
-
\frac{L}{2}
\|\mathbf u-\mathbf w^t\|^2.
\end{align}
Using $
\Phi_{\widehat{\lambda}^{\,t}}(\mathbf w^t)
=
\Psi(\mathbf w^t)-\varepsilon_{\Psi,t}
$
and
$\nabla\Phi_{\widehat{\lambda}^{\,t}}(\mathbf w^t)
=\bar{\boldsymbol{\xi}}^t$
gives \eqref{eq:approx_fairness_weak_subgradient}.
Equation~\eqref{eq:implemented_fairness_unbiased} follows from the
conditional independence of the selection and gradient
mini-batches.
\end{proof}

\begin{lemma}[Moreau-Envelope Properties]
\label{lem:moreau_properties}
Let $\mathcal L$ be $\kappa$-weakly convex and let
$\mu\in(0,1/\kappa)$. Then the proximal mapping is single-valued,
$\mathcal L_\mu$ is continuously differentiable, and
\begin{equation}
\nabla\mathcal L_\mu(\mathbf w)
=
\frac{1}{\mu}
\left(
\mathbf w-
\operatorname{prox}_{\mu\mathcal L}(\mathbf w)
\right).
\end{equation}
Moreover,
$
\frac{1}{\mu}
\left(
\mathbf w-
\operatorname{prox}_{\mu\mathcal L}(\mathbf w)
\right)
\in
\partial\mathcal L
\left(
\operatorname{prox}_{\mu\mathcal L}(\mathbf w)
\right)$.
\end{lemma}
\vspace{-0.2cm}
\subsection{Technical Lemmas for the Convergence Analysis}
\label{subsec:key_lemmas}
We next present three technical lemmas that form the backbone of
the convergence analysis. Lemma~\ref{lem:lemma4} establishes a
one-step descent inequality for the Moreau envelope
$\mathcal L_\mu$. Lemma~\ref{lem:lemma5} controls the error caused
by evaluating the graph-coupling direction using stale and
compressed neighbor-model copies. Lemma~\ref{lem:lemma6} provides
a cumulative bound on the resulting communication residual under
the finite-horizon stepsize and trigger schedule. Together, these
results separate the optimization, fairness-selection, stochastic,
and communication-induced errors.

Let $\mathcal H_t
:=
\sigma\left(
\mathcal F_t,
\{B_{i,\mathrm{sel}}^t\}_{i=1}^n
\right)$ denote the information available after the fairness-selection
mini-batches have been observed. Hence,
$\widehat{\lambda}^{\,t}$ and $\varepsilon_{\Psi,t}$ are
$\mathcal H_t$-measurable.

Define the stochastic-gradient noise by
\begin{equation}
q_i^t
:=
\left(1+\rho\widehat{\lambda}_i^t\right)
\left(
g_i^t-\nabla F_i(w_i^t)
\right),
\qquad
\mathbf q^t:=(q_1^t,\ldots,q_n^t).
\label{eq:amplified_gradient_noise}
\end{equation}
By Assumption~\ref{ass:variance}, $\mathbb E
\left[
\mathbf q^t
\,\middle|\,
\mathcal H_t
\right] = \mathbf 0,
$
and
\begin{equation}
\mathbb E
\left[
\|\mathbf q^t\|^2
\,\middle|\,
\mathcal H_t
\right]
\leq
(1+\rho)^2
\sum_{i=1}^n\sigma_i^2
=
n(1+\rho)^2\bar\sigma^2.
\label{eq:amplified_noise_variance}
\end{equation}
Define the graph-coupling approximation error by
\begin{equation}
d_i^t
:=
r_i^t
-
\nabla_{w_i}\mathcal R_{\mathcal G}(\mathbf w^t),
\qquad
\mathbf d^t:=(d_1^t,\ldots,d_n^t).
\label{eq:graph_direction_error}
\end{equation}
Using the definition of $r_i^t$, this error can be written as
\begin{equation}
d_i^t
=
\alpha
\sum_{j\in\mathcal N_i}
a_{ij}
\left(
w_j^t-\widehat w_{j\to i}^t
\right).
\label{eq:graph_direction_error_expanded}
\end{equation}
The local direction used by Algorithm~\ref{alg:dmflsq} therefore
admits the decomposition
\begin{equation}
h_i^t
=
V_i^t+q_i^t+d_i^t,
\qquad
\mathbf h^t
=
\mathbf V^t+\mathbf q^t+\mathbf d^t.
\label{eq:update_direction_decomposition}
\end{equation}
Define the average weighted communication residual as
\begin{equation}
E^t
:=
\frac{1}{n}
\sum_{i=1}^{n}
\sum_{j\in\mathcal N_i}
a_{ij}
\left\|
\mathbf w_j^t-\widehat{\mathbf w}_{j\to i}^t
\right\|^2.
\label{eq:communication_residual}
\end{equation}

\begin{lemma}[One-Step Moreau-Envelope Descent]
\label{lem:lemma4}
Let
$0<\mu<\frac{1}{\kappa},
\kappa=(1+\rho)L.$
Under Assumptions~\ref{ass:smoothness},
\ref{ass:variance}, and \ref{ass:bounded_direction}, the iterates
generated by Algorithm~\ref{alg:dmflsq} satisfy
\begin{align}
\mathbb E
\left[
\mathcal L_\mu(\mathbf w^{t+1})
\right]
&\leq
\mathbb E
\left[
\mathcal L_\mu(\mathbf w^t)
\right]
-
\frac{1-\kappa\mu}{4}
\gamma_t
\mathbb E
\left[
\left\|
\nabla\mathcal L_\mu(\mathbf w^t)
\right\|^2
\right]
\nonumber\\
&\quad
+
\frac{\rho\gamma_t}{\mu}
\mathbb E[\varepsilon_{\Psi,t}]
+
\left(
\frac{\gamma_t}{1-\kappa\mu}
+
\frac{\gamma_t^2}{\mu}
\right)
\mathbb E
\left[
\|\mathbf d^t\|^2
\right]
\nonumber\\
&\quad
+
\frac{\gamma_t^2}{\mu}H^2
+
\frac{n(1+\rho)^2\gamma_t^2}{2\mu}
\bar\sigma^2.
\label{eq:moreau_one_step_descent}
\end{align}
\end{lemma}

\begin{IEEEproof}
See Appendix~\ref{app:lem4}.
\end{IEEEproof}

\begin{lemma}[Communication-Residual Recursion]
\label{lem:lemma5}
Under Assumption~\ref{ass:compression} and the event-trigger
condition in Algorithm~\ref{alg:dmflsq}, there exist constants
$\chi_0,\chi_1,\chi_2>0$, independent of $t$ and $T$, such that
$\chi_0\omega\in(0,1]$ and
\begin{align}
\mathbb E[\mathcal E^{t+1}]
&\leq
(1-\chi_0\omega)\mathbb E[\mathcal E^t]
+
\chi_1\gamma_t^2
\frac{1}{n}
\mathbb E\|\mathbf h^t\|^2
+
\chi_2\vartheta_t^2.
\label{eq:communication_residual1}
\end{align}
Moreover, the graph-direction approximation error defined in
\eqref{eq:graph_direction_error} satisfies
\begin{equation}
\|\mathbf d^t\|^2
\leq
nK_{\mathcal G}\mathcal E^t,
\label{eq:graph_error_residual_relation}
\end{equation}
where $K_{\mathcal G}>0$ depends only on the graph weights and the
coupling parameter $\alpha$.
\end{lemma}

\begin{IEEEproof}
See Appendix~\ref{app:lem5}.
\end{IEEEproof}

\noindent Using the decomposition
$\mathbf h^t=\mathbf V^t+\mathbf q^t+\mathbf d^t$,
the conditional zero-mean property of $\mathbf q^t$, and
Assumption~\ref{ass:bounded_direction}, we obtain
\begin{equation}
\frac{1}{n}\mathbb E\|\mathbf h^t\|^2
\leq
B_h
+
2K_{\mathcal G}\mathbb E[\mathcal E^t],
\label{eq:local_direction_moment}
\end{equation}
where
$B_h:=2H^2/n+(1+\rho)^2\bar{\sigma}^2$.
Indeed,
$\mathbb E\|\mathbf h^t\|^2
=\mathbb E\|\mathbf V^t+\mathbf d^t\|^2
+\mathbb E\|\mathbf q^t\|^2
\leq
2H^2+2\mathbb E\|\mathbf d^t\|^2
+n(1+\rho)^2\bar{\sigma}^2$,
and \eqref{eq:graph_error_residual_relation} completes the bound.
Substituting \eqref{eq:local_direction_moment} into
\eqref{eq:communication_residual1} gives
\begin{align}
\mathbb E[\mathcal E^{t+1}]
&\leq
\left(
1-\chi_0\omega
+
2\chi_1K_{\mathcal G}\gamma_t^2
\right)
\mathbb E[\mathcal E^t]
+
\chi_1B_h\gamma_t^2
\nonumber\\
&\quad+
\chi_2\vartheta_t^2.
\label{eq:closed_residual_recursion}
\end{align}

\begin{lemma}[Average Communication-Residual Bound]
\label{lem:lemma6}
Suppose the conditions of Lemma~\ref{lem:lemma5} hold, and let
$\gamma_t=\gamma/\sqrt{T}$ and
$\vartheta_t=\vartheta_0/\sqrt{t+1}$. If $\gamma>0$ is sufficiently
small such that
$
2\chi_1K_{\mathcal G}\gamma^2
\leq
\frac{\chi_0\omega}{2},
$ then the closed recursion in
\eqref{eq:closed_residual_recursion} satisfies
\begin{equation}
\mathbb E[\mathcal E^{t+1}]
\leq
\left(1-\frac{\chi_0\omega}{2}\right)
\mathbb E[\mathcal E^t]
+
\chi_1B_h\gamma_t^2
+
\chi_2\vartheta_t^2.
\label{eq:contractive_residual_recursion}
\end{equation}
Consequently,
\begin{equation}
\frac{1}{T}
\sum_{t=0}^{T-1}
\mathbb E[\mathcal E^t]
\leq
\frac{2}{\chi_0\omega T}
\left[
\mathbb E[\mathcal E^0]
+
\chi_1B_h\gamma^2
+
\chi_2\vartheta_0^2(1+\log T)
\right].
\label{eq:communication_residual_average}
\end{equation}
\end{lemma}
\begin{IEEEproof}
See Appendix~\ref{app:lem6}.
\end{IEEEproof}
\vspace{-0.5cm}
\subsection{Convergence Rate}
Since the objective is defined over client-specific models, stationarity
is measured with respect to the stacked variable
$\mathbf w=(w_1,\ldots,w_n)$ rather than a network-average model.
For $0<\mu<1/\kappa$, define the normalized Moreau-envelope
stationarity measure
\[
\mathcal S_\mu^t
=
\frac{1}{n}
\left\|
\nabla\mathcal L_\mu(\mathbf w^t)
\right\|^2.
\]
By Lemma~\ref{lem:moreau_properties}, a small value of $\mathcal S_\mu^t$ implies that
$\mathbf w^t$ is close to a point that is nearly stationary for the
original nonsmooth and possibly nonconvex objective. Thus, unlike
consensus-based decentralized optimization, the analysis
characterizes approximate stationarity of the full personalized
objective rather than of a consensus-restricted objective.

\noindent\textbf{Remark on fairness selection.}
For $\Lambda=\Delta_n$, if the scalar max-consensus step satisfies
$K_\Psi\geq\operatorname{diam}(\mathcal G)$, all clients agree on
the same empirical fairness maximizer
$\widehat{\lambda}^{\,t}$. This agreement does not necessarily
imply that the selected component maximizes the population fairness
envelope. The resulting discrepancy is quantified by
$\varepsilon_{\Psi,t}$ and controlled through
Assumption~\ref{ass:selection_accuracy}. Approximate or delayed
fairness consensus can similarly be incorporated into an enlarged
fairness-selection gap.

\begin{theorem}[Convergence of DMFL-SQ]
\label{thm:main}
Suppose Assumptions~\ref{ass:smoothness}--
\ref{ass:trigger} hold. Assume also that
$\mathcal L^\star:=\inf_{\mathbf w}\mathcal L(\mathbf w)>-\infty$.
Let $\mu\in(0,1/\kappa)$, where $\kappa=(1+\rho)L$, and run
DMFL-SQ for $T$ iterations with
$\gamma_t=\gamma/\sqrt{T}$ and
$\vartheta_t=\vartheta_0/\sqrt{t+1}$.

If $\gamma>0$ is sufficiently small such that
$2\chi_1K_{\mathcal G}\gamma^2\leq\chi_0\omega/2$, then
\begin{align}
\frac{1}{T}
\sum_{t=0}^{T-1}
\mathbb E[\mathcal S_\mu^t]
&\leq
\frac{4\Delta_\mu^0}
{n(1-\kappa\mu)\gamma\sqrt{T}}
+
\frac{4\rho B_\Psi}
{n\mu(1-\kappa\mu)\sqrt{T}}
\nonumber\\
&\quad
+
\frac{4\gamma H^2}
{n\mu(1-\kappa\mu)\sqrt{T}}
+
\frac{2\gamma(1+\rho)^2\bar{\sigma}^2}
{\mu(1-\kappa\mu)\sqrt{T}}
\nonumber\\
&\quad
+
\frac{8K_{\mathcal G}}
{\chi_0\omega T(1-\kappa\mu)}
\left(
\frac{1}{1-\kappa\mu}
+
\frac{\gamma}{\mu\sqrt{T}}
\right)
\nonumber\\
&\qquad\qquad\times
\left[
\mathbb E[\mathcal E^0]
+
\chi_1B_h\gamma^2
+
\chi_2\vartheta_0^2(1+\log T)
\right],
\label{eq:main_convergence_bound}
\end{align}
where
$\Delta_\mu^0:=\mathcal L_\mu(\mathbf w^0)-\mathcal L^\star$ and
$\mathcal S_\mu^t=n^{-1}
\|\nabla\mathcal L_\mu(\mathbf w^t)\|^2$.
All constants in \eqref{eq:main_convergence_bound} are independent
of $T$.
\end{theorem}
Consequently,
\begin{equation}
\frac{1}{T}
\sum_{t=0}^{T-1}
\mathbb E[\mathcal S_\mu^t]
=
\mathcal O(T^{-1/2})
+
\mathcal O\left(\frac{\log T}{T}\right).
\label{eq:main_rate}
\end{equation}
\begin{IEEEproof}
See Appendix~\ref{appendix:convergence}.
\end{IEEEproof}
\noindent Theorem~\ref{thm:main} establishes an
$\mathcal O(T^{-1/2})$ rate in expected squared
Moreau-envelope stationarity for the nonsmooth and 
nonconvex objective. The communication terms decay as
$\mathcal O((1+\log T)/T)$ and are therefore lower order.
The bound explicitly captures the effects of stochastic gradients,
empirical fairness selection, compression, and event-triggered
neighbor-model communication.

\vspace{-0.2cm}
\section{Simulation Results}
\label{sec:sim_res}
We evaluate \textsc{DMFL-SQ} in decentralized, heterogeneous, and
communication-constrained settings using CIFAR-10 and the real MUSMET EEG
dataset. For CIFAR-10, we use the standard 50k/10k train/test split and
partition the data across \(n=20\) clients using a Dirichlet distribution
with \(\alpha_{\rm Dir}=0.1\), inducing strong non-iid heterogeneity. Each
client trains a CNN with two convolutional layers followed by two fully
connected layers, and clients communicate over an undirected peer-to-peer
graph; a ring topology is used by default unless otherwise specified.

For MUSMET, we consider a 4-class EEG-based emotion recognition task
corresponding to aggressive, happy, relax, and sad states~\cite{musmet2025}.
The dataset contains synchronized EEG/audio recordings from \(20\)
musicians. Labels are obtained from the recording protocol, with each
\texttt{.xdf} file corresponding to one emotional condition. EEG signals
are bandpass-filtered over 1--40 Hz, segmented into 2-second
non-overlapping windows, and represented using Welch PSD features over
standard frequency bands. Each musician is treated as one client and trains
a 3-layer MLP with hidden dimensions $[256,128]$. For each musician and emotional-condition recording, the non-overlapping
2-s windows are split chronologically, with the first 80\% used for
training and the remaining 20\% for testing. Feature normalization is
computed from the training portion only and then applied to both sets.

All methods are implemented in PyTorch. DMFL-SQ uses mini-batch stochastic gradient descent with one local
update per communication round. For computational efficiency, the experiments use independent
fixed-size fairness-selection and gradient mini-batches, each of size
$32$. This fixed-batch implementation is used for finite-horizon empirical
evaluation and is not claimed to verify the asymptotic selection-gap condition in Assumption~\ref{ass:selection_accuracy}. The learning rate is $\gamma=0.1$. Unless otherwise stated, we use
$\alpha=0.4$, $\rho=0.5$, a sparsity ratio $k/P=0.01$, an
$8$-bit quantizer, and the event-trigger schedule
$\vartheta_t=0.5/\sqrt{t+1}$. The experimental compressor is
$\mathcal C=Q_8\circ\operatorname{Top}_k$, where $Q_8$ is
the support-preserving uniform 8-bit quantizer. Since
$\|Q_8(u)-u\|^2\leq
\frac{k}{2(2^8-1)^2}\|u\|^2$ and
$\|\operatorname{Top}_k(v)\|^2\geq(k/P)\|v\|^2$, 
the composite compressor satisfies Assumption~\ref{ass:compression} by taking
$\omega=\frac{k}{P}
\left(1-\frac{k}{2(2^8-1)^2}\right)>0$;
for CIFAR-10, this gives
$\omega\approx8.37\times10^{-3}$.

All communication-efficiency plots report
$\overline{B}_{\mathrm{client}}$, which accounts for sparse-quantized
model-copy messages, coordinate-index or mask overhead, quantizer scales,
message headers, graph-degree effects, and scalar fairness-consensus
messages; we use $b_q=8$, $b_{\mathrm{scale}}=32$,
$b_{\mathrm{loss}}=32$, $b_{\mathrm{hdr}}=32$, and
$b_{\mathrm{trig}}=0$ bits, with silence indicating an inactive trigger.

\begin{figure*}[t]
\begin{minipage}{0.33\linewidth}\centering
  \resizebox{0.95\columnwidth}{!}{%
    \includegraphics[scale=0.25]{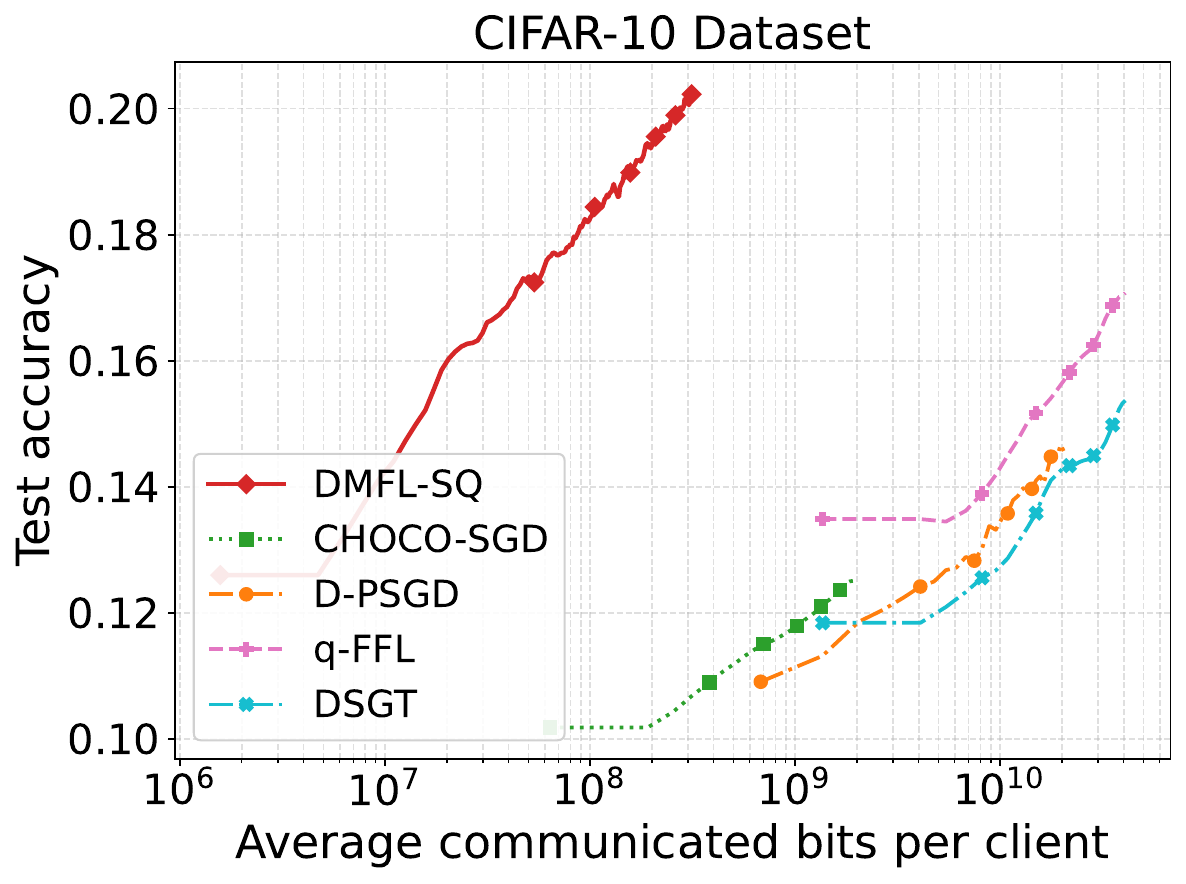}}
  \caption{{Test accuracy versus average communicated bits per client on the CIFAR-10 dataset.
}}
  \label{fig:comm_tradeoff}
\end{minipage}
\begin{minipage}{0.33\linewidth}\centering
  \resizebox{0.95\columnwidth}{!}{%
    \includegraphics[scale=0.25]{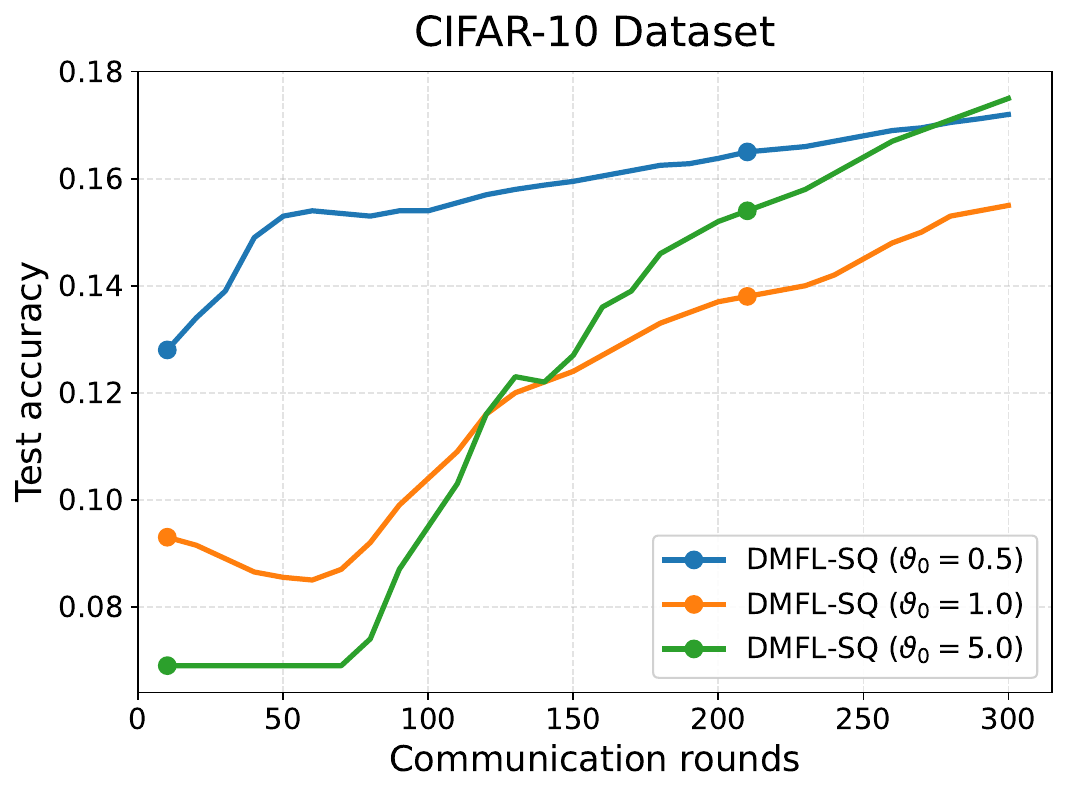}}
  \caption{{Ablation on the event-trigger threshold $\vartheta_0$ on the CIFAR-10 dataset.
}}
  \label{fig:ablation_eta} 
\end{minipage}
\begin{minipage}{0.32\linewidth}\centering
  \resizebox{0.95\columnwidth}{!}{%
    \includegraphics[scale=0.25]{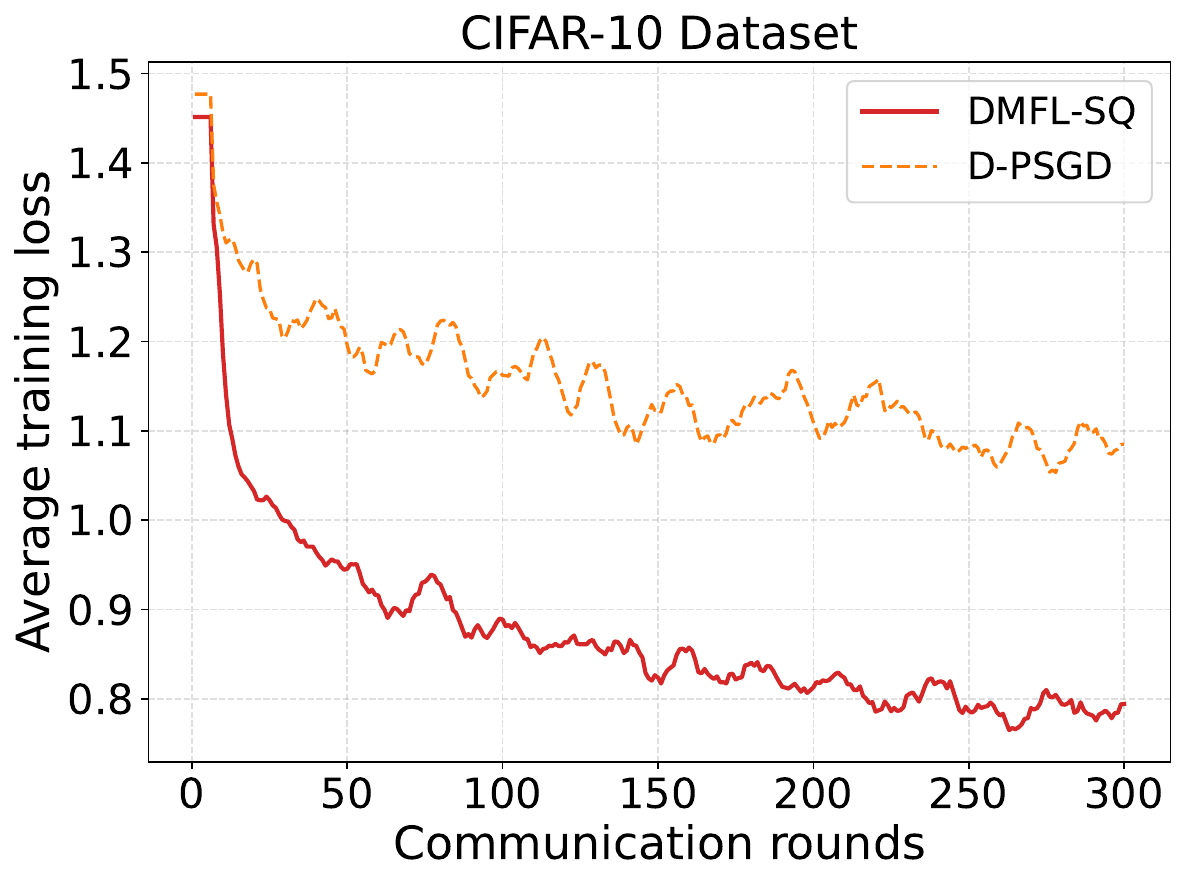}}
  \caption{{Average training loss versus communication rounds on the CIFAR-10 dataset.}}
  \label{fig:consensus}
\end{minipage}
\end{figure*}

\begin{figure*}[t]
\begin{minipage}{0.33\linewidth}\centering
  \resizebox{0.95\columnwidth}{!}{%
    \includegraphics[scale=0.23]{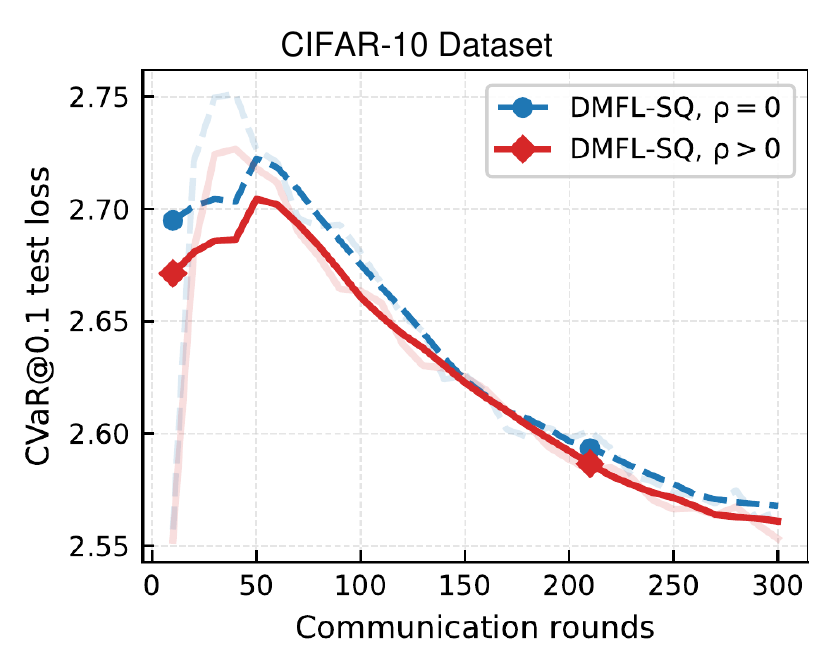}}
  \caption{{Effect of fairness regularization on CIFAR-10 tail risk (CVaR@0.1).
}}
  \label{fig:fair1}
\end{minipage}
\begin{minipage}{0.33\linewidth}\centering
  \resizebox{0.95\columnwidth}{!}{%
    \includegraphics[scale=0.25]{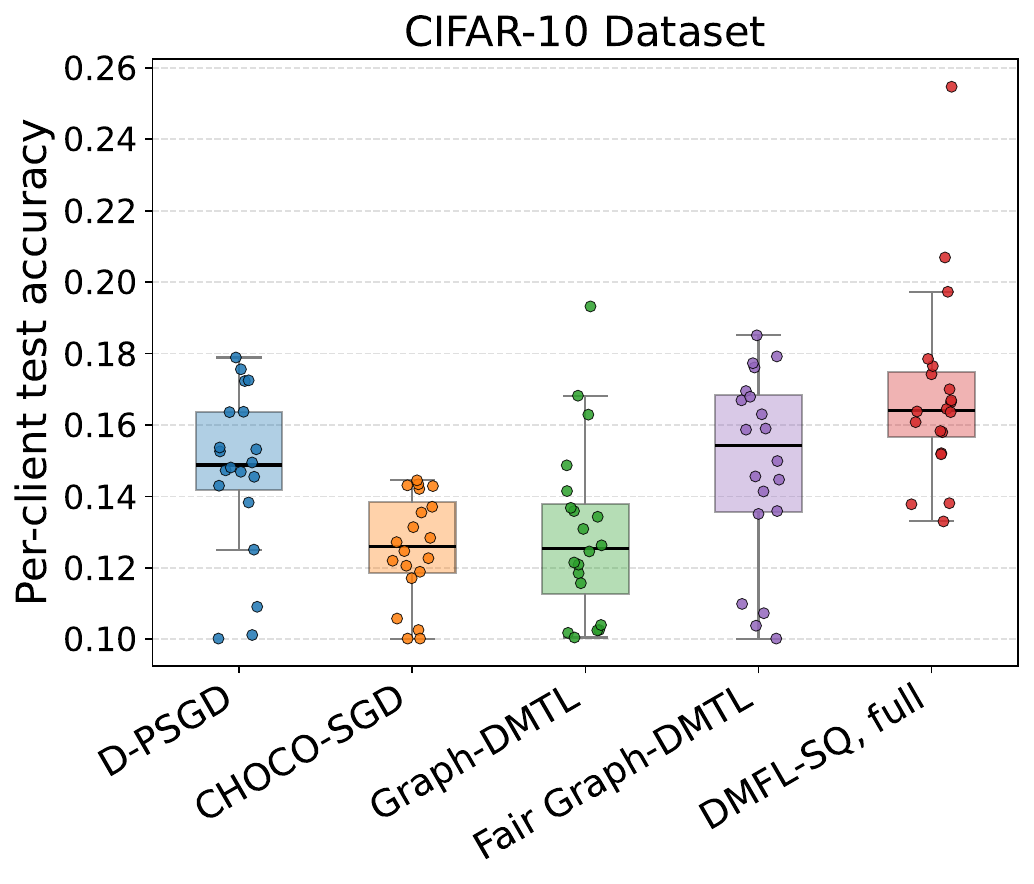}}
  \caption{{Client-level test-accuracy distribution on non-iid CIFAR-10.
}}
  \label{fig:client} 
\end{minipage}
\begin{minipage}{0.32\linewidth}\centering
  \resizebox{0.95\columnwidth}{!}{%
    \includegraphics[scale=0.25]{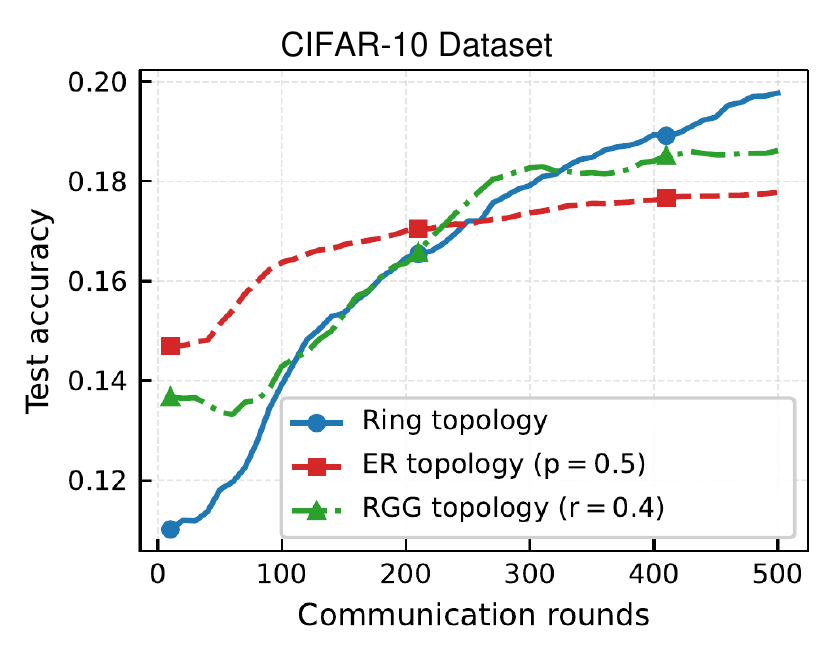}}
  \caption{{Effect of network topology on DMFL-SQ (CIFAR-10, non-iid).
}}
  \label{fig:topo_ring_er}
\end{minipage}
\end{figure*}

\begin{figure*}[t]
\begin{minipage}{0.33\linewidth}\centering
  \resizebox{0.95\columnwidth}{!}{%
    \includegraphics[scale=0.25]{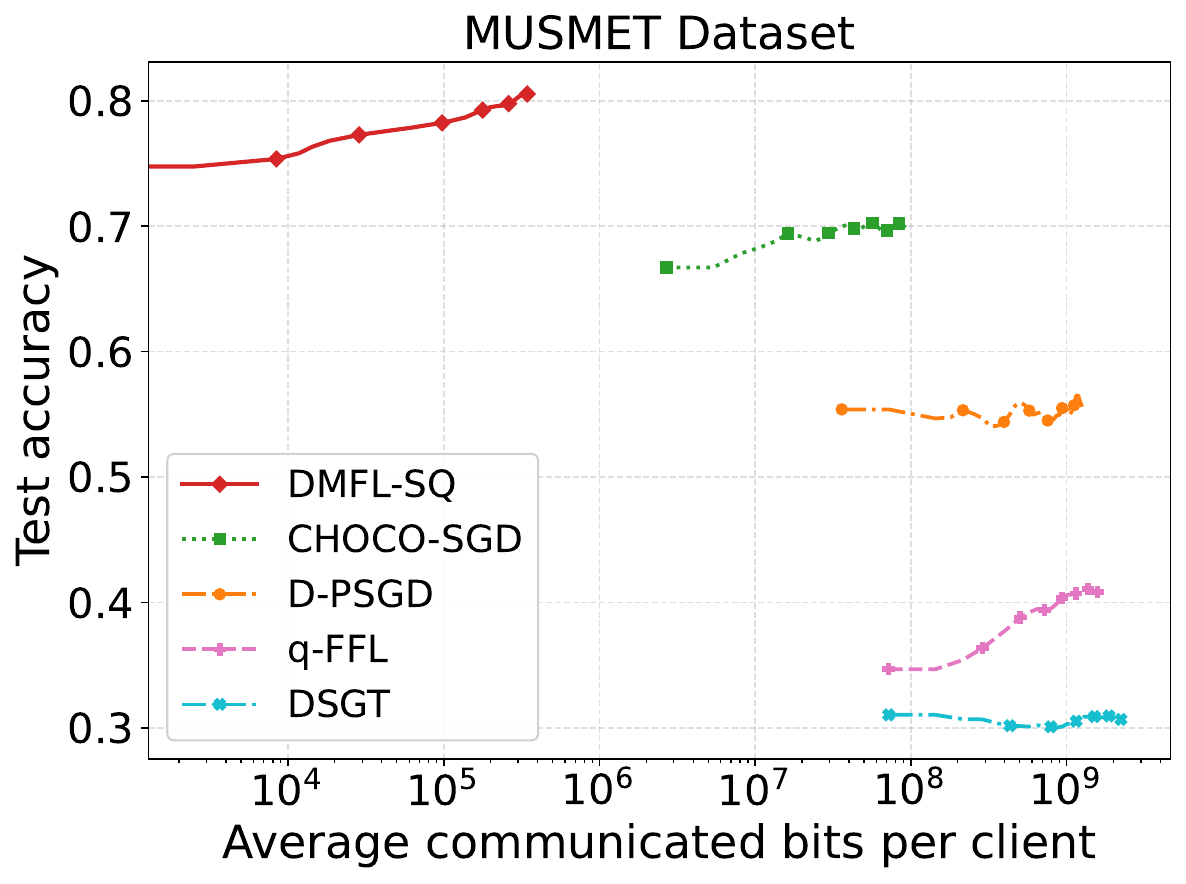}}
  \caption{{Test accuracy versus average communicated bits per client on the MUSMET dataset.
}}
  \label{fig:comm_tradeoff_m}
\end{minipage}
\begin{minipage}{0.33\linewidth}\centering
  \resizebox{0.95\columnwidth}{!}{%
    \includegraphics[scale=0.25]{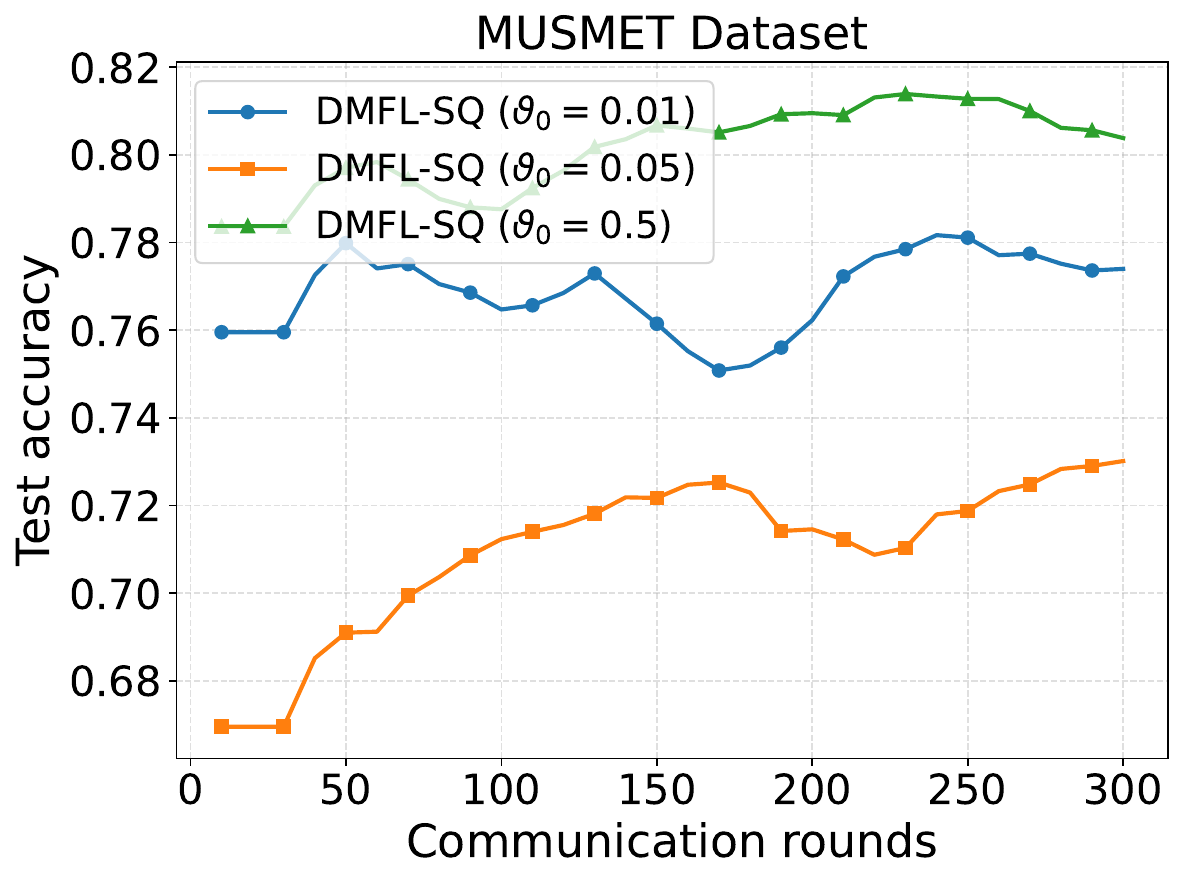}}
  \caption{{Ablation on the event-trigger threshold $\vartheta_0$ on the MUSMET dataset.
}}
  \label{fig:ablation_eta_m} 
\end{minipage}
\begin{minipage}{0.32\linewidth}\centering
  \resizebox{0.95\columnwidth}{!}{%
    \includegraphics[scale=0.25]{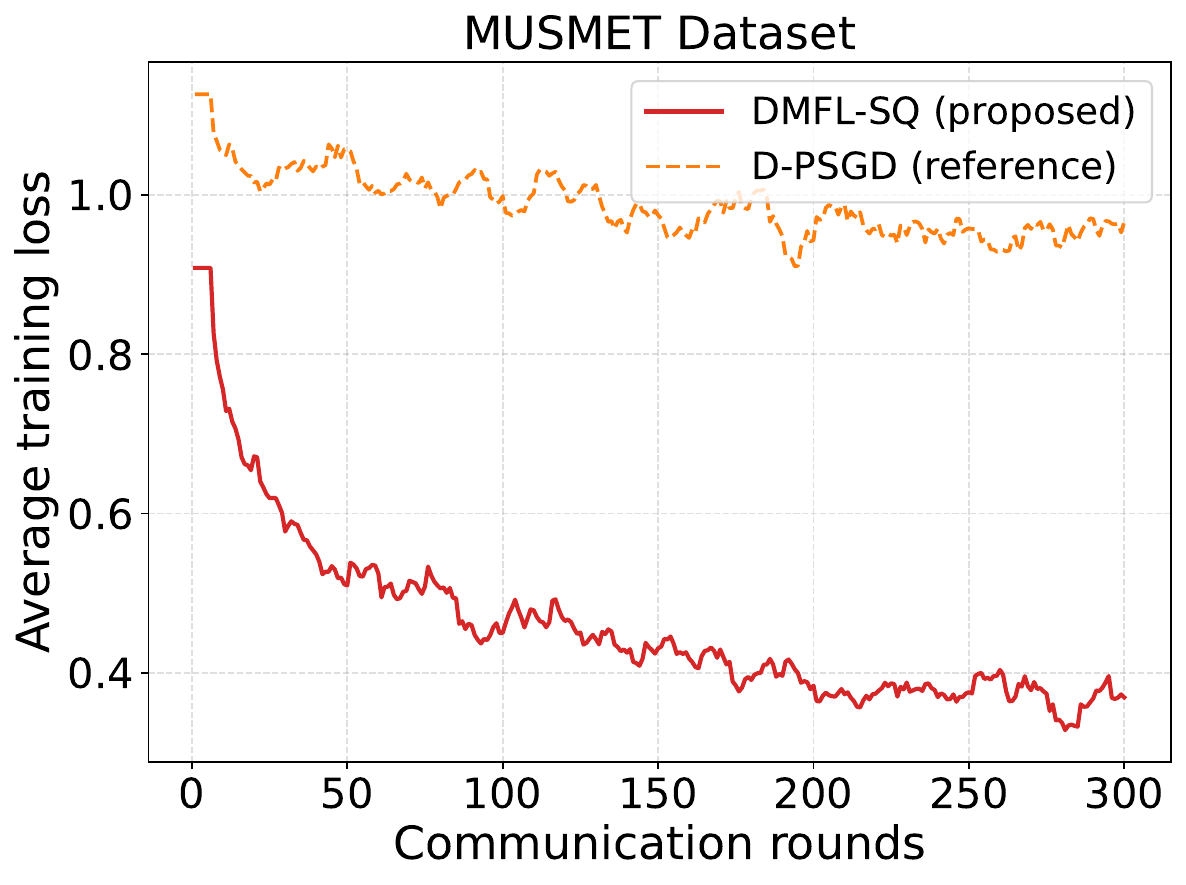}}
  \caption{{Average training loss versus communication rounds on the MUSMET dataset.}}
  \label{fig:consensus_m}
\end{minipage}
\end{figure*}

\begin{figure*}[t]
\begin{minipage}{0.33\linewidth}\centering
  \resizebox{0.95\columnwidth}{!}{%
    \includegraphics[scale=0.25]{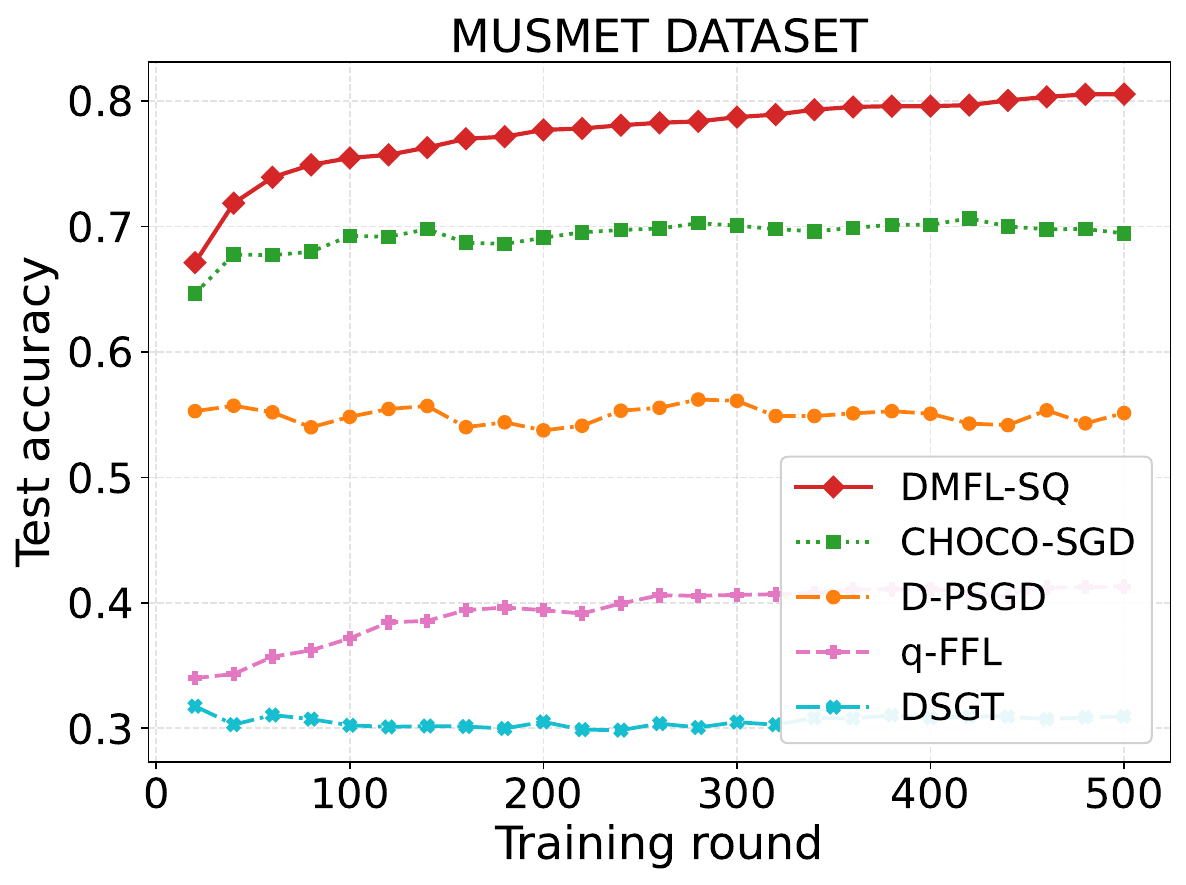}}
  \caption{{Test accuracy versus training rounds on the MUSMET dataset.
}}
  \label{fig:train_tradeoff_m2}
\end{minipage}
\begin{minipage}{0.33\linewidth}\centering
  \resizebox{0.95\columnwidth}{!}{%
    \includegraphics[scale=0.25]{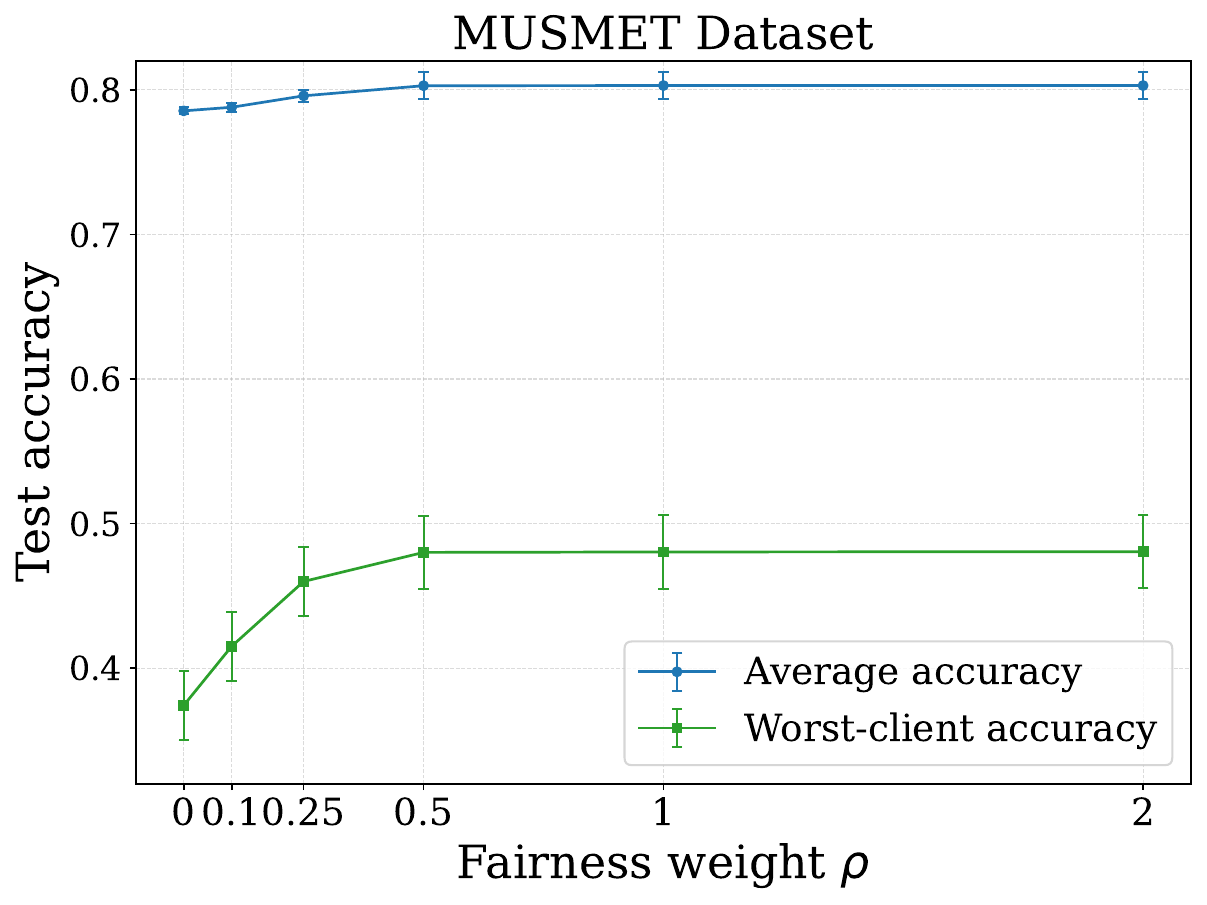}}
  \caption{Sensitivity of DMFL-SQ to the fairness weight $\rho$ on the MUSMET dataset. }
  \label{fig:sensitivity_rho_m} 
\end{minipage}
\begin{minipage}{0.32\linewidth}\centering
  \resizebox{0.95\columnwidth}{!}{%
    \includegraphics[scale=0.25]{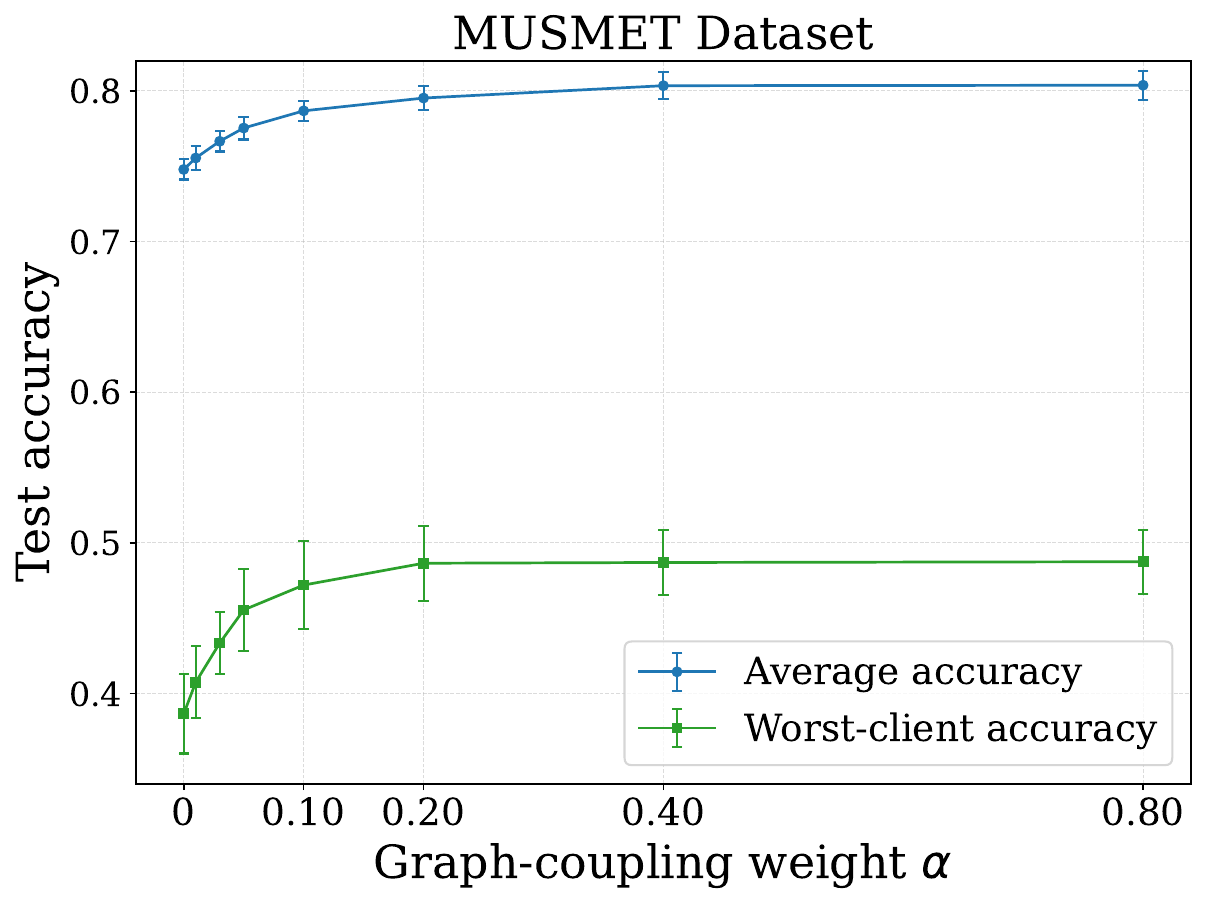}}
\caption{Sensitivity of DMFL-SQ to the graph-coupling weight $\alpha$ on the MUSMET dataset.  }
   \label{fig:sensitivity_alpha_m}
\end{minipage}
\end{figure*}

The proposed DMFL-SQ  algorithm is compared with the following benchmark algorithms:
\begin{itemize}
\item {D-PSGD} performs local stochastic-gradient updates followed by
dense model mixing among neighboring clients~\cite{lian2017can}.
\item \textsc{CHOCO-SGD} is a
communication-efficient decentralized method that replaces dense
neighbor exchanges with compressed gossip updates~\cite{pmlr-v97-koloskova19a}.
\item \textsc{DSGT}
augments decentralized optimization with an auxiliary gradient-tracking
variable to better estimate the network-wide descent direction~\cite{jiaqizhang19}.
\item {q-FFL} is server-based  and is included primarily as a fairness-aware
reference~\cite{li2019fair}. Its communication cost includes both uplink and downlink:
each client transmits and receives one full-precision $P$-dimensional
vector per round, giving $B_{\rm qFFL,client}=2TPb_{\rm fp}$ with
$b_{\rm fp}=32$.
\end{itemize}
For baseline parity, all methods use the same MUSMET client partitions,
model architecture, mini-batch size, and five random
seeds, while all decentralized methods use the same communication graph. The principal baseline settings are: D-PSGD with learning
rate $0.035$; CHOCO-SGD with learning rate $0.018$, consensus step $0.25$,
sparsity ratio $0.1$, 8-bit quantization, and weight decay $5\times10^{-5}$;
q-FFL with learning rate $0.18$, $q=1$, $L=2$, and gradient clipping at
$0.5$; and DSGT with stepsize $0.045/\sqrt{t+1}$ and tracker
clipping at $1.2$. All methods use the same fixed training horizon $T=500$ and perform one local optimization update per training round, while DSGT
requires one additional gradient evaluation per client for tracker
initialization.
For each realized graph, we set
$K_\Psi=\operatorname{diam}(\mathcal G)$ and use the realized
$|\mathcal E|$ in the communication accounting. With $n=20$ and
$b_{\rm loss}=32$, the fairness-consensus cost is
$B_{\Psi,t}=74K_\Psi|\mathcal E|$ bits per round. For the ring,
ER $(p=0.5)$, and RGG $(r=0.4)$ realizations,
$(|\mathcal E|,K_\Psi)$ is $(20,10)$, $(95,2)$, and $(60,4)$,
respectively, corresponding to $740$, $703$, and $888$
bits/client/round. The reported total communication uses the actual
per-round trigger indicators in $B_{\rm model,t}$.

\begin{table*}[t]
\centering
\caption{Accuracy-fairness-communication comparison on the MUSMET dataset. Performance metrics are reported as mean $\pm$ 95\% CI across $R = 5$ seeds.}
\label{tab:ablation}
\begin{adjustbox}{max width=\textwidth}
\begin{tabular}{lcccccc}
\hline
Method & Avg. Acc. $\uparrow$ & Worst Acc. $\uparrow$ & Bottom-10\% Acc. $\uparrow$ & Acc. Std $\downarrow$ & Jain's Index $\uparrow$ & Bits/Client $\downarrow$ \\
\hline
DSGT & $0.3132\pm 0.0541$ & $0.0501 \pm 0.0586$ & $0.0610 \pm 0.0592$ & $0.1981 \pm 0.0666$ & $0.7103 \pm 0.1719$ & $1998.3$ M \\
q-FFL & $0.4114\pm 0.0155$ & $0.0716 \pm 0.1173$ & $0.1091 \pm 0.0849	$ & $0.2025 \pm 0.0880	$ & $0.7983 \pm 0.1385$ & $1618.0$ M \\
D-PSGD & $		0.5716\pm 0.0556$ & $0.1957 \pm 0.1004$ & $0.2236 \pm 0.0932$ & $0.2156 \pm 0.0304$ & $0.8713 \pm 0.0520$ & $1438.6$ M \\
CHOCO-SGD & $0.6928\pm 0.0342$ & $0.1935\pm 0.1205$ & $0.2320 \pm 0.0665	$ & $	0.2163 \pm 0.0256$ & $	0.9074 \pm 0.0128	$ & $97.4$ M \\
\textbf{DMFL-SQ} & $0.8030\pm0.0096$ & $0.4810 \pm 0.0225$ & $0.5160 \pm 0.0260$ & $	0.1210 \pm 0.0058$ & $0.9779 \pm 0.0047$ & $0.6$ M \\
\hline
\end{tabular}
\end{adjustbox}
\end{table*}

\begin{table*}[t]
\centering
\caption{Component ablation of DMFL-SQ on MUSMET Dataset.
Performance metrics are reported as mean $\pm$ 95\% CI across $R = 5$ seeds.}
\label{tab:component_ablation}
\begin{tabular}{lcccccc}
\hline
Variant & Avg. Acc. $\uparrow$ & Worst Acc. $\uparrow$ & Bottom-10\% Acc. $\uparrow$ & Jain's Index $\uparrow$ & Bits/Client $\downarrow$ \\
\hline
No fairness ($\rho=0$) & $0.7869 \pm 0.0024$ & $0.3743\pm 0.0241$ & $0.4065 \pm 0.0241$ & $0.9521\pm 0.0041$ & 0.3 M \\
No graph coupling ($\alpha=0$) & $0.7474 \pm 0.0157$ & $0.3862\pm 0.0263$ & $0.4206\pm 0.0207$ & $0.9554\pm 0.0042$ & 0.2 M \\
No event-triggering & $0.8044 \pm 0.0211$ & $0.4397\pm0.0178$ & $0.4850\pm  0.0211$ & $0.9748 \pm 0.0046$ & 4.0 M\\
No sparsification & $0.8187 \pm 0.0169$ & $0.4504\pm 0.0218$ & $0.4905 \pm 0.0211$ & $0.9737 \pm 0.0043$ & 2.2 M\\
No quantization & $0.8074 \pm 0.0067$ & $0.4405 \pm 0.0236$ & $0.4820\pm 0.0210$ & $0.9746 \pm 0.0045$ & 5.3 M\\
\textbf{DMFL-SQ} & $0.8030 \pm 0.0096$ & $0.4810 \pm 0.0225 $ & $0.5160\pm 0.0260$ & $0.9779\pm 0.0047$ & $0.6$ M\\
\hline
\end{tabular}
\end{table*}

Fig.~\ref{fig:comm_tradeoff} shows the test accuracy versus average
communicated bits per client on non-iid CIFAR-10. DMFL-SQ reaches higher
accuracy with substantially fewer communicated bits than the dense decentralized baselines D-PSGD and DSGT. Compared with CHOCO-SGD, DMFL-SQ
achieves better final accuracy while remaining communication efficient,
illustrating the benefit of combining personalization, fairness, compression,
and event-triggered communication.

\noindent\textbf{Remark.} The modest CIFAR-10 accuracies reflect the challenging setting considered here: $20$ clients, a strongly non-iid Dirichlet split with $\alpha_{\mathrm{Dir}}=0.1$, peer-to-peer communication only, and no central server or global aggregation. Thus, the results should be interpreted as communication-fairness-accuracy trade-offs under severe decentralized heterogeneity rather than as centralized CIFAR-10 benchmarks. We further complement this setting with MUSMET experiments on naturally heterogeneous EEG data.

Fig.~\ref{fig:ablation_eta} evaluates the sensitivity of DMFL-SQ to the
initial event-trigger threshold $\vartheta_0$. Smaller thresholds trigger
more frequent communication and lead to faster early convergence, while
larger thresholds reduce communication but slow the initial learning
progress. The results show that DMFL-SQ remains stable across different
threshold choices, with larger thresholds eventually catching up when more
communication rounds are allowed.

Fig.~\ref{fig:consensus} shows the average training loss on non-iid
CIFAR-10. DMFL-SQ achieves a faster decrease and a lower final loss than
D-PSGD, suggesting improved optimization behavior under heterogeneous
client data. The curve remains stable despite the use of sparse-quantized
event-triggered communication.

Fig.~\ref{fig:fair1} shows the effect of fairness regularization on the
CVaR@0.1 test loss under non-iid CIFAR-10. After an initial transient, both
variants reduce tail risk as training progresses. The fairness-aware
variant with \(\rho>0\) remains slightly below the \(\rho=0\) variant for
most rounds, indicating that the agnostic mixture term helps reduce
upper-tail client loss.

Fig.~\ref{fig:client} reports the client-level test-accuracy
distribution on non-iid CIFAR-10. Each point corresponds to one client.
DMFL-SQ shifts the lower tail of the distribution upward compared with the
baselines.

\noindent\textbf{Remark:} The PAC-Bayes result is stated for bounded losses \(\ell_i(w_i;z)\in[0,1]\), as is standard in PAC-Bayes analysis. In the experiments, worst-client loss and CVaR@0.1 are computed using raw cross-entropy loss and are reported only as empirical diagnostic metrics; they are not used as the bounded loss in Theorem~\ref{thm:pacbayes}. To apply the theorem directly to cross-entropy, one may use the clipped and normalized loss $\ell_{i,B}(w_i;z)=B^{-1}\min\{\ell_i^{\rm CE}(w_i;z),B\}, $ which lies in \([0,1]\).

Fig.~\ref{fig:topo_ring_er} evaluates the sensitivity of DMFL-SQ to the
communication topology. Ring, Erdős--Rényi (ER), and random geometric graph
(RGG) topologies all show stable learning, indicating that DMFL-SQ is not
restricted to a specific graph structure. The different transient behaviors
reflect how graph connectivity affects neighbor-copy propagation and
graph-regularized coupling under non-iid data.

Fig.~\ref{fig:comm_tradeoff_m} shows the communication-accuracy
trade-off on the MUSMET dataset. Similar to the CIFAR-10 results in Fig.~\ref{fig:comm_tradeoff}, DMFL-SQ attains the highest accuracy
among the compared methods while requiring fewer communicated bits than dense decentralized baselines, showing a favorable accuracy-communication trade-off on this naturally
heterogeneous real-world dataset.

Fig.~\ref{fig:ablation_eta_m} shows the sensitivity of DMFL-SQ to the
initial event-trigger threshold $\vartheta_0$ on MUSMET. Across the tested
thresholds, the method exhibits stable learning behavior, showing that the
event-triggered mechanism remains effective on naturally heterogeneous EEG
data.

Fig.~\ref{fig:consensus_m} shows the average training loss on MUSMET.
DMFL-SQ converges faster and reaches a much lower loss than D-PSGD,
demonstrating stable optimization under naturally heterogeneous EEG data
despite sparse-quantized event-triggered communication.

Fig.~\ref{fig:train_tradeoff_m2} shows that DMFL-SQ consistently attains the highest test accuracy and converges stably compared with the considered baselines on the MUSMET dataset. 

Figs.~\ref{fig:sensitivity_rho_m} and~\ref{fig:sensitivity_alpha_m} further show that increasing the fairness weight $\rho$ mainly improves worst-client accuracy while preserving average accuracy, whereas stronger graph coupling $\alpha$ improves both metrics, with diminishing gains beyond moderate values. These results support the robustness of DMFL-SQ to its key hyperparameters.

Table~\ref{tab:ablation} shows that DMFL-SQ achieves the highest
average accuracy, worst-client accuracy, bottom-10\% accuracy, and Jain's
index and lowest standard deviation of client accuracies on the heterogeneous MUSMET dataset. Moreover, DMFL-SQ substantially
reduces the communication cost compared with both dense decentralized
baselines and compressed CHOCO-SGD, demonstrating a favorable
accuracy-fairness-communication trade-off. 

Table~\ref{tab:component_ablation} shows that each component contributes to
the overall performance of DMFL-SQ. The DMFL-SQ algorithm achieves the highest worst-client and bottom-10\%
accuracies, as well as the highest Jain's index, while requiring only
$0.6$ Mbits per client. Although removing sparsification or quantization
slightly improves the average accuracy, it substantially increases the
communication cost and reduces the tail-client performance. Similarly,
removing the fairness or graph-coupling terms degrades the worst-client
and bottom-10\% accuracies. These results demonstrate that the DMFL-SQ design provides the most favorable overall balance among average
accuracy, client fairness, and communication efficiency.

\vspace{-0.3cm}
\section{Conclusion}
\label{sec:concl}
We presented DMFL-SQ, a decentralized multi-task fair learning framework with sparsified, quantized, and event-triggered communication. The algorithm unifies personalization, fairness, and communication efficiency within a single decentralized optimization paradigm in the \emph{non-convex} regime. Our theoretical results establish an
$\mathcal{O}(T^{-1/2})$ rate in expected squared Moreau-envelope
stationarity, while explicitly accounting for compression, network topology, and triggering effects. Extensive experiments demonstrate that DMFL-SQ achieves a favorable
accuracy-fairness-communication trade-off under the considered
heterogeneous settings while reducing communication cost compared to existing baselines. 
\vspace{-0.3cm}
\appendices
\section{Proof of Theorem~\ref{thm:pacbayes} \& Corollary~\ref{cor:decentralization_compression}}
\label{app:pacbayes}

We first recall the population and empirical agnostic mixture risks: $\Psi(\mathbf w)
= \sup_{\lambda\in\Lambda}
\sum_{i=1}^n \lambda_i F_i(w_i)$, $\widehat\Psi(\mathbf w)
= \sup_{\lambda\in\Lambda}
\sum_{i=1}^n \lambda_i \widehat F_i(w_i)$. For a fixed mixture vector $\lambda\in\Lambda$, define
$ \Phi_\lambda(\mathbf w)
:= \sum_{i=1}^n \lambda_i F_i(w_i)$, $\widehat\Phi_\lambda(\mathbf w) :=
\sum_{i=1}^n \lambda_i \widehat F_i(w_i)$.

\paragraph{Case 1: $\Lambda=\Delta_n$}
Since $\lambda\mapsto \Phi_\lambda(\mathbf w)$ is linear over the simplex
$\Delta_n$, its supremum is attained at a vertex. Hence,
\[
\Psi(\mathbf w)
=
\max_{1\le i\le n} F_i(w_i),
\qquad
\widehat\Psi(\mathbf w)
=
\max_{1\le i\le n} \widehat F_i(w_i).
\]
For each $\mathbf w$, let $
i^\star(\mathbf w) \in \arg\max_{1\le i\le n} F_i(w_i)$
be a measurable selection of a worst-case client. Define the augmented
posterior and prior over pairs $(\mathbf w,i)$ by $\widetilde Q(d\mathbf w,di)
= Q(d\mathbf w)\delta_{i^\star(\mathbf w)}(di)$,
$\widetilde \Pi(d\mathbf w,di)
= \Pi(d\mathbf w)\pi(di)$, where $\pi(i)=1/n$ is the uniform distribution over clients. Then
\[
\mathbb E_{(\mathbf w,i)\sim\widetilde Q}
\!\left[F_i(w_i)\right]
=
\mathbb E_{\mathbf w\sim Q}
\!\left[\Psi(\mathbf w)\right].
\]
Moreover, since
$\widehat F_{i^\star(\mathbf w)}
\!\left(w_{i^\star(\mathbf w)}\right)
\le
\max_{1\le i\le n}\widehat F_i(w_i)
=
\widehat\Psi(\mathbf w)$,
we have
\[
\mathbb E_{(\mathbf w,i)\sim\widetilde Q}
\!\left[\widehat F_i(w_i)\right]
\le
\mathbb E_{\mathbf w\sim Q}
\!\left[\widehat\Psi(\mathbf w)\right].
\]
The KL divergence of the augmented posterior satisfies
\[
\mathrm{KL}(\widetilde Q\|\widetilde \Pi)
=
\mathrm{KL}(Q\|\Pi)+\ln n.
\] 
For every fixed pair $(\mathbf{w},i)$,
Hoeffding's lemma gives
\begin{equation}
\mathbb{E}_{\mathcal{D}}
\exp\left(
\tau_{\mathrm{PB}}
\left[
F_i(w_i)-\widehat{F}_i(w_i)
\right]
\right)
\leq
\exp\left(
\frac{\tau_{\mathrm{PB}}^2}{8m_i}
\right)
\leq
\exp\left(
\frac{\tau_{\mathrm{PB}}^2}{8m_{\min}}
\right).\nonumber
\end{equation}
Integrating with respect to the data-independent augmented prior
$\widetilde{\Pi}$ and applying the PAC--Bayes change-of-measure
inequality together with Markov's inequality yields, with probability
at least $1-\delta$, simultaneously for all augmented posteriors
$\widetilde{Q}$,
]
\vspace{-0.2cm}
\begin{align}
\mathbb{E}_{(\mathbf w,i)\sim\widetilde Q}
\!\left[F_i(w_i)\right]
&\leq
\mathbb{E}_{(\mathbf w,i)\sim\widetilde Q}
\!\left[\widehat F_i(w_i)\right]
+
\frac{
\mathrm{KL}(\widetilde Q\Vert\widetilde\Pi)
+\ln(1/\delta)
}{
\tau_{\mathrm{PB}}
}
\nonumber\\
&+
\frac{\tau_{\mathrm{PB}}}{8m_{\min}}.
\end{align}
This proves the first claim.
\paragraph{Case 2: General compact $\Lambda\subseteq\Delta_n$} Let $\mathcal N_\epsilon(\Lambda,\|\cdot\|_1)\subseteq\Lambda$
be a fixed deterministic $\epsilon$-net of $\Lambda$ in the
$\ell_1$ norm, constructed independently of the observed datasets. Since
$\ell_i(w_i;z)\in[0,1]$, we have
$F_i(w_i),\widehat F_i(w_i)\in[0,1]$. Therefore, for any fixed
$\mathbf w$, both maps $\lambda\mapsto \Phi_\lambda(\mathbf w)$, $\lambda\mapsto \widehat \Phi_\lambda(\mathbf w)$ are $1$-Lipschitz with respect to $\|\cdot\|_1$. Indeed, for any
$\lambda,\lambda'\in\Lambda$,
\[
|\Phi_\lambda(\mathbf w)-\Phi_{\lambda'}(\mathbf w)|
\le
\sum_{i=1}^n|\lambda_i-\lambda_i'|F_i(w_i)
\le
\|\lambda-\lambda'\|_1,
\]
and similarly $|\widehat \Phi_\lambda(\mathbf w)-\widehat \Phi_{\lambda'}(\mathbf w)| \le \|\lambda-\lambda'\|_1$.

For each $\mathbf w$, let
$\lambda^\star(\mathbf w)\in
\arg\max_{\lambda\in\Lambda}\Phi_\lambda(\mathbf w)$
be a measurable selection. Using a fixed deterministic
tie-breaking rule, choose
$v_\epsilon(\mathbf w)\in
\mathcal N_\epsilon(\Lambda,\|\cdot\|_1)$ such that $
\|\lambda^\star(\mathbf w)-v_\epsilon(\mathbf w)\|_1
\leq \epsilon$. By the Lipschitz property,
\[
\Psi(\mathbf w)
=
\Phi_{\lambda^\star(\mathbf w)}(\mathbf w)
\le
\Phi_{v_\varepsilon(\mathbf w)}(\mathbf w)+\varepsilon.
\]
Also, $\widehat \Phi_{v_\varepsilon(\mathbf w)}(\mathbf w)
\le \widehat\Psi(\mathbf w)$. Define the augmented posterior and prior over pairs $(\mathbf w,v)$ by
$ \widetilde Q_\varepsilon(d\mathbf w,dv)
= Q(d\mathbf w)\delta_{v_\varepsilon(\mathbf w)}(dv)$,
$\widetilde \Pi_\varepsilon(d\mathbf w,dv)
= \Pi(d\mathbf w)\pi_\varepsilon(dv)$, where $\pi_\varepsilon$ is the uniform distribution over
$\mathcal N_\varepsilon(\Lambda,\|\cdot\|_1)$. Then
\[
\mathrm{KL}(\widetilde Q_\varepsilon\|\widetilde \Pi_\varepsilon)
=
\mathrm{KL}(Q\|\Pi)
+
\ln|\mathcal N_\varepsilon(\Lambda,\|\cdot\|_1)|.
\]
For each
$v\in\mathcal N_\varepsilon(\Lambda,\|\cdot\|_1)$, define
$m_v:=m_\lambda|_{\lambda=v}$. Since the $\varepsilon$-net is contained
in $\Lambda$, the definition of $m_{\mathrm{eff}}$ gives
$m_v\ge m_{\mathrm{eff}}$.

For every fixed pair $(\mathbf{w},v)$, where
$v\in\mathcal{N}_{\epsilon}(\Lambda,\|\cdot\|_1)$,
\begin{align}
\Phi_v(\mathbf{w})-\widehat{\Phi}_v(\mathbf{w})
=
\sum_{i=1}^{n}\sum_{k=1}^{m_i}
\frac{v_i}{m_i}
\left[
F_i(w_i)-\ell_i(w_i;z_{i,k})
\right].\nonumber
\end{align}
Since the losses lie in $[0,1]$ and the client samples are mutually
independent, Hoeffding's lemma gives
\begin{align}
&\mathbb{E}_{\mathcal{D}}
\exp\left(
\tau_{\mathrm{PB}}
\left[
\Phi_v(\mathbf{w})-\widehat{\Phi}_v(\mathbf{w})
\right]
\right)
\nonumber\\
&\qquad\leq
\exp\left(
\frac{\tau_{\mathrm{PB}}^2}{8}
\sum_{i=1}^{n}\frac{v_i^2}{m_i}
\right)
=
\exp\left(
\frac{\tau_{\mathrm{PB}}^2}{8m_v}
\right),
\end{align}
where $m_v=
\left(
\sum_{i=1}^{n}\frac{v_i^2}{m_i}
\right)^{-1}$.
Since $m_v\geq m_{\mathrm{eff}}$, integrating with respect to
$\widetilde{\Pi}_{\epsilon}$ and applying the PAC-Bayes
change-of-measure inequality and Markov's inequality yields, with
probability at least $1-\delta$, simultaneously for all augmented
posteriors $\widetilde{Q}_{\epsilon}$,
\begin{align}
\mathbb{E}_{(\mathbf w,v)\sim\widetilde Q_\epsilon}
\!\left[\Phi_v(\mathbf w)\right]
\leq {}&
\mathbb{E}_{(\mathbf w,v)\sim\widetilde Q_\epsilon}
\!\left[\widehat\Phi_v(\mathbf w)\right]
+
\frac{\tau_{\mathrm{PB}}}{8m_{\mathrm{eff}}} \nonumber\\
&+
\frac{
\mathrm{KL}(\widetilde Q_\epsilon
\Vert\widetilde\Pi_\epsilon)
+\ln(1/\delta)
}{
\tau_{\mathrm{PB}}
}.
\end{align}

This proves the second claim.
\vspace{-0.2cm}
\subsection{Proof of Corollary~\ref{cor:decentralization_compression}}
\label{app:cor_decentralization_compression}

The PAC-Bayes bound in Theorem~\ref{thm:pacbayes} is
algorithm-agnostic. We now quantify the one-step perturbation of the
empirical agnostic mixture risk caused by evaluating the graph
regularizer using stale and compressed neighbor copies.

First, we establish a Lipschitz property of the empirical fairness
envelope. For any two stacked model collections
$\mathbf w=(w_1,\ldots,w_n)$ and
$\mathbf v=(v_1,\ldots,v_n)$, we have
\begin{align}
\left|
\widehat\Psi(\mathbf w)-\widehat\Psi(\mathbf v)
\right|
&=
\left|
\sup_{\lambda\in\Lambda}
\sum_{i=1}^n\lambda_i\widehat F_i(w_i)
-
\sup_{\lambda\in\Lambda}
\sum_{i=1}^n\lambda_i\widehat F_i(v_i)
\right|
\nonumber\\
&\le
\sup_{\lambda\in\Lambda}
\left|
\sum_{i=1}^n
\lambda_i
\left(
\widehat F_i(w_i)-\widehat F_i(v_i)
\right)
\right|
\nonumber\\
&\le
G_{\mathrm{lip}}
\sup_{\lambda\in\Lambda}
\sum_{i=1}^n
\lambda_i\|w_i-v_i\|.
\label{eq:empirical_envelope_lipschitz_1}
\end{align}
Since $\Lambda\subseteq\Delta_n$,
\begin{align}
\sup_{\lambda\in\Lambda}
\sum_{i=1}^n
\lambda_i\|w_i-v_i\|
&\le
\max_{i\in[n]}\|w_i-v_i\|
\le
\|\mathbf w-\mathbf v\|.
\end{align}
Consequently,
\begin{align}
\left|
\widehat\Psi(\mathbf w)-\widehat\Psi(\mathbf v)
\right|
\le
G_{\mathrm{lip}}\|\mathbf w-\mathbf v\|.
\label{eq:empirical_envelope_lipschitz}
\end{align}

Recall that the implemented direction can be decomposed as
$\mathbf h^t
=
\mathbf V^t+\mathbf q^t+\mathbf d^t$,
where $\mathbf d^t
=
\mathbf r^t-\nabla\mathcal R_{\mathcal G}(\mathbf w^t)$ is the graph-coupling approximation error. Define the exact-neighbor
shadow update
\begin{align}
\mathbf w_{\mathrm{ex}}^{t+1}
:=
\mathbf w^t-\gamma_t
\left(
\mathbf h^t-\mathbf d^t
\right).
\label{eq:exact_neighbor_shadow_update}
\end{align}
Thus, $\mathbf w_{\mathrm{ex}}^{t+1}$ starts from the same current
iterate $\mathbf w^t$ and uses the same stochastic gradients and
fairness selection as DMFL-SQ, but evaluates the graph-coupling
direction using the exact current neighbor models.

Since the DMFL-SQ update is $\mathbf w^{t+1}
=
\mathbf w^t-\gamma_t\mathbf h^t$,
we obtain
\begin{align}
\mathbf w^{t+1}
-
\mathbf w_{\mathrm{ex}}^{t+1}
=
-\gamma_t\mathbf d^t.
\label{eq:shadow_update_difference}
\end{align}
Applying~\eqref{eq:empirical_envelope_lipschitz} gives
\begin{align}
\left|
\widehat\Psi(\mathbf w^{t+1})
-
\widehat\Psi(\mathbf w_{\mathrm{ex}}^{t+1})
\right|
&\le
G_{\mathrm{lip}}\gamma_t\|\mathbf d^t\|.
\label{eq:one_step_empirical_risk_difference}
\end{align}
Taking expectations and using Jensen's inequality yields
\begin{align}
\left|
\mathbb E\!\left[
\widehat\Psi(\mathbf w^{t+1})
\right]
-
\mathbb E\!\left[
\widehat\Psi(\mathbf w_{\mathrm{ex}}^{t+1})
\right]
\right|
&\le
G_{\mathrm{lip}}\gamma_t
\mathbb E\|\mathbf d^t\|
\nonumber\\
&\le
G_{\mathrm{lip}}\gamma_t
\left(
\mathbb E\|\mathbf d^t\|^2
\right)^{1/2}.
\label{eq:expected_one_step_empirical_risk}
\end{align}

The graph-direction approximation bound gives $\|\mathbf d^t\|^2
\le
nK_{\mathcal G}\mathcal E^t$,
where $\mathcal E^t
=
\frac{1}{n}
\sum_{i=1}^n
\sum_{j\in\mathcal N_i}
a_{ij}
\left\|
w_j^t-\widehat w_{j\to i}^t
\right\|^2$.
Therefore,
\begin{align}
&\left|
\mathbb E\!\left[
\widehat\Psi(\mathbf w^{t+1})
\right]
-
\mathbb E\!\left[
\widehat\Psi(\mathbf w_{\mathrm{ex}}^{t+1})
\right]
\right|
\le
G_{\mathrm{lip}}\gamma_t
\sqrt{
nK_{\mathcal G}\,
\mathbb E[\mathcal E^t]
}.
\label{eq:one_step_residual_risk_bound}
\end{align}

Let $\tau\sim\mathrm{Unif}\{0,\ldots,T-1\}$, independently
of the algorithmic randomness. Averaging
\eqref{eq:one_step_residual_risk_bound} over
$t=0,\ldots,T-1$ gives
\begin{align}
&\left|
\mathbb E\!\left[
\widehat\Psi(\mathbf w^{\tau+1})
\right]
-
\mathbb E\!\left[
\widehat\Psi(\mathbf w_{\mathrm{ex}}^{\tau+1})
\right]
\right|
\le
\frac{
G_{\mathrm{lip}}\sqrt{nK_{\mathcal G}}
}{T}
\sum_{t=0}^{T-1}
\gamma_t
\sqrt{\mathbb E[\mathcal E^t]}.
\label{eq:uniform_shadow_average}
\end{align}

Under Assumption~\ref{ass:trigger},
$\gamma_t=\gamma/\sqrt{T}$. Hence, by the Cauchy-Schwarz
inequality,
\begin{align}
\label{eq:uniform_shadow_residual}
\left|
\mathbb E\!\left[
\widehat\Psi(\mathbf w^{\tau+1})
\right]
-
\mathbb E\!\left[
\widehat\Psi(\mathbf w_{\mathrm{ex}}^{\tau+1})
\right]
\right|
\le
\frac{
G_{\mathrm{lip}}\gamma\sqrt{nK_{\mathcal G}}
}{\sqrt{T}}
\left(
\frac{1}{T}
\sum_{t=0}^{T-1}
\mathbb E[\mathcal E^t]
\right)^{1/2}.
\end{align}

By Lemma~\ref{lem:lemma6},
\begin{align}
\frac{1}{T}
\sum_{t=0}^{T-1}
\mathbb E[\mathcal E^t]
\le
\frac{2}{\chi_0\omega T}
\left(
\mathbb E[\mathcal E^0]
+
\chi_1B_h\gamma^2
+
\chi_2\vartheta_0^2(1+\log T)
\right).
\label{eq:cor_current_residual_bound}
\end{align}
Substituting~\eqref{eq:cor_current_residual_bound} into
\eqref{eq:uniform_shadow_residual} yields
\begin{align}
&\left|
\mathbb E\!\left[
\widehat\Psi(\mathbf w^{\tau+1})
\right]
-
\mathbb E\!\left[
\widehat\Psi(\mathbf w_{\mathrm{ex}}^{\tau+1})
\right]
\right|
\nonumber\\
&\qquad\le
\frac{
G_{\mathrm{lip}}\gamma
\sqrt{2nK_{\mathcal G}}
}{
\sqrt{\chi_0\omega}\,T
}
\left(
\mathbb E[\mathcal E^0]
+
\chi_1B_h\gamma^2
+
\chi_2\vartheta_0^2(1+\log T)
\right)^{1/2}.
\label{eq:cor_final_shadow_bound}
\end{align}
Therefore, the one-step empirical agnostic mixture-risk perturbation
caused by stale, sparse, quantized, and event-triggered neighbor-model
exchange satisfies
\[
\left|
\mathbb E\!\left[
\widehat\Psi(\mathbf w^{\tau+1})
\right]
-
\mathbb E\!\left[
\widehat\Psi(\mathbf w_{\mathrm{ex}}^{\tau+1})
\right]
\right|
=
\mathcal O\!\left(
\frac{\sqrt{\log T}}{T}
\right),
\]
and hence vanishes asymptotically as $T\to\infty$.
This proves the corollary.
\hfill$\square$
\vspace{-0.2cm}
\section{Proof of Theorem~\ref{thm:main} (Convergence Analysis)}
\label{appendix:convergence}
We prove Theorem~\ref{thm:main} using the notation introduced in Section~\ref{sec:convergence}. For each $t$, define the proximal point and the associated
Moreau-envelope gradient as
$
\widetilde{\mathbf w}^{t}
:=
\operatorname{prox}_{\mu\mathcal L}(\mathbf w^t)$, $
\mathbf s^t
:=
\frac{1}{\mu}
\left(
\mathbf w^t-\widetilde{\mathbf w}^{t}
\right)
=
\nabla\mathcal L_{\mu}(\mathbf w^t)$.
\vspace{-0.2cm}
\subsection{Proof of Lemma~\ref{lem:lemma4}}
\label{app:lem4}
Combining Lemma~\ref{lem:approx_fairness_direction} with the
$L$-weak convexity of $\sum_{i=1}^n F_i(w_i)$ and the convexity
of $\mathcal R_{\mathcal G}$ gives, for every
$\mathbf u\in\mathbb R^{nd}$,
\begin{align}
\mathcal L(\mathbf u)
&\geq
\mathcal L(\mathbf w^t)
+
\left\langle
\mathbf V^t,\mathbf u-\mathbf w^t
\right\rangle
-
\frac{\kappa}{2}
\|\mathbf u-\mathbf w^t\|^2
-
\rho\varepsilon_{\Psi,t},
\label{eq:objective_approx_direction}
\end{align}
where $\kappa=(1+\rho)L$ and $\mathbf V^t$ is defined in
Assumption~\ref{ass:bounded_direction}.
Applying~\eqref{eq:objective_approx_direction} with
$\mathbf u=\widetilde{\mathbf w}^{\,t}$ gives
\begin{align}
\mathcal L(\widetilde{\mathbf w}^{\,t})
&\geq
\mathcal L(\mathbf w^t)
+
\left\langle
\mathbf V^t,
\widetilde{\mathbf w}^{\,t}-\mathbf w^t
\right\rangle
\nonumber\\
&\quad
-
\frac{\kappa}{2}
\left\|
\widetilde{\mathbf w}^{\,t}-\mathbf w^t
\right\|^2
-
\rho\varepsilon_{\Psi,t}.
\label{eq:proof_approx_weak}
\end{align}
On the other hand, by the definition of the proximal point and by
using $\mathbf w^t$ as a feasible candidate,
\begin{equation}
\mathcal L(\widetilde{\mathbf w}^{\,t})
+
\frac{1}{2\mu}
\left\|
\widetilde{\mathbf w}^{\,t}-\mathbf w^t
\right\|^2
\leq
\mathcal L(\mathbf w^t).
\label{eq:proof_prox_optimality}
\end{equation}
Therefore,
\begin{equation}
\mathcal L(\mathbf w^t)
-
\mathcal L(\widetilde{\mathbf w}^{\,t})
\geq
\frac{1}{2\mu}
\left\|
\widetilde{\mathbf w}^{\,t}-\mathbf w^t
\right\|^2.
\label{eq:proof_prox_decrease}
\end{equation}

Rearranging \eqref{eq:proof_approx_weak} and using
\eqref{eq:proof_prox_decrease}, we obtain
\begin{align}
\left\langle
\mathbf V^t,
\mathbf w^t-\widetilde{\mathbf w}^{\,t}
\right\rangle
&\geq
\left(
\frac{1}{2\mu}
-
\frac{\kappa}{2}
\right)
\left\|
\mathbf w^t-\widetilde{\mathbf w}^{\,t}
\right\|^2
-
\rho\varepsilon_{\Psi,t}.
\end{align}
Substituting
$\mathbf w^t-\widetilde{\mathbf w}^{\,t}
=\mu\mathbf s^t$ yields
\begin{equation}
\left\langle
\mathbf V^t,\mathbf s^t
\right\rangle
\geq
\frac{1-\kappa\mu}{2}
\|\mathbf s^t\|^2
-
\frac{\rho}{\mu}
\varepsilon_{\Psi,t}.
\label{eq:proof_direction_stationarity}
\end{equation}

The vector form of the algorithmic update is
\begin{equation}
\mathbf w^{t+1}
=
\mathbf w^t-\gamma_t\mathbf h^t,
\qquad
\mathbf h^t
=
\mathbf V^t+\mathbf q^t+\mathbf d^t.
\label{eq:proof_vector_update}
\end{equation}
Using $\widetilde{\mathbf w}^{\,t}$ as a feasible candidate in the
definition of $\mathcal L_\mu(\mathbf w^{t+1})$, we have
\begin{align}
\mathcal L_\mu(\mathbf w^{t+1})
&\leq
\mathcal L(\widetilde{\mathbf w}^{\,t})
+
\frac{1}{2\mu}
\left\|
\widetilde{\mathbf w}^{\,t}
-
\mathbf w^{t+1}
\right\|^2
\nonumber\\
&=
\mathcal L(\widetilde{\mathbf w}^{\,t})
+
\frac{1}{2\mu}
\left\|
-\mu\mathbf s^t+\gamma_t\mathbf h^t
\right\|^2
\nonumber\\
&=
\mathcal L_\mu(\mathbf w^t)
-
\gamma_t
\left\langle
\mathbf s^t,\mathbf h^t
\right\rangle
+
\frac{\gamma_t^2}{2\mu}
\|\mathbf h^t\|^2.
\label{eq:proof_envelope_upper}
\end{align}

Recall that $\mathcal H_t$ contains the history
$\mathcal F_t$ and the fairness-selection mini-batches at round
$t$. Conditional on $\mathcal H_t$, the quantities
$\mathbf V^t$, $\mathbf d^t$, $\mathbf s^t$, and
$\varepsilon_{\Psi,t}$ are measurable, while
$\mathbb E
\left[
\mathbf q^t
\,\middle|\,
\mathcal H_t
\right]
=
\mathbf 0.$
Consequently,
$
\mathbb E
\left[
\left\langle
\mathbf s^t,\mathbf q^t
\right\rangle
\,\middle|\,
\mathcal H_t
\right]
=
0$. Taking the conditional expectation of
\eqref{eq:proof_envelope_upper} and using
\eqref{eq:proof_direction_stationarity} gives
\begin{align}
\mathbb E
\left[
\mathcal L_\mu(\mathbf w^{t+1})
\,\middle|\,
\mathcal H_t
\right]
&\leq
\mathcal L_\mu(\mathbf w^t)
-
\frac{1-\kappa\mu}{2}
\gamma_t
\|\mathbf s^t\|^2
+
\frac{\rho\gamma_t}{\mu}
\varepsilon_{\Psi,t}
\nonumber\\
&
-
\gamma_t
\left\langle
\mathbf s^t,\mathbf d^t
\right\rangle
+
\frac{\gamma_t^2}{2\mu}
\mathbb E
\left[
\|\mathbf h^t\|^2
\,\middle|\,
\mathcal H_t
\right].
\label{eq:proof_conditional_descent}
\end{align}

By Young's inequality,
\begin{equation}
-
\left\langle
\mathbf s^t,\mathbf d^t
\right\rangle
\leq
\frac{1-\kappa\mu}{4}
\|\mathbf s^t\|^2
+
\frac{1}{1-\kappa\mu}
\|\mathbf d^t\|^2.
\label{eq:proof_young_stale}
\end{equation}
Combining this with \eqref{eq:proof_conditional_descent} yields
\begin{align}
\mathbb E
\left[
\mathcal L_\mu(\mathbf w^{t+1})
\,\middle|\,
\mathcal H_t
\right]
&\leq
\mathcal L_\mu(\mathbf w^t)
-
\frac{1-\kappa\mu}{4}
\gamma_t
\|\mathbf s^t\|^2
+
\frac{\rho\gamma_t}{\mu}
\varepsilon_{\Psi,t}
\nonumber\\
&\quad
+
\frac{\gamma_t}{1-\kappa\mu}
\|\mathbf d^t\|^2
+
\frac{\gamma_t^2}{2\mu}
\mathbb E
\left[
\|\mathbf h^t\|^2
\,\middle|\,
\mathcal H_t
\right].
\label{eq:proof_after_young}
\end{align}

Using
$\mathbf h^t=\mathbf V^t+\mathbf d^t+\mathbf q^t$ and the
conditional zero-mean property of $\mathbf q^t$, we obtain
\begin{align}
\mathbb E
\left[
\|\mathbf h^t\|^2
\,\middle|\,
\mathcal H_t
\right]
&=
\|\mathbf V^t+\mathbf d^t\|^2
+
\mathbb E
\left[
\|\mathbf q^t\|^2
\,\middle|\,
\mathcal H_t
\right]
\nonumber\\
&\leq
2\|\mathbf V^t\|^2
+
2\|\mathbf d^t\|^2
+
\mathbb E
\left[
\|\mathbf q^t\|^2
\,\middle|\,
\mathcal H_t
\right].
\label{eq:proof_direction_second_moment}
\end{align}
Furthermore, by Assumption~\ref{ass:variance},
\begin{align}
\mathbb E
\left[
\|\mathbf q^t\|^2
\,\middle|\,
\mathcal H_t
\right]
&=
\sum_{i=1}^n
\left(1+\rho\widehat{\lambda}_i^t\right)^2
\mathbb E
\left[
\left\|
g_i^t-\nabla F_i(w_i^t)
\right\|^2
\,\middle|\,
\mathcal H_t
\right]
\nonumber\\
&\leq
(1+\rho)^2
\sum_{i=1}^n\sigma_i^2
\nonumber\\
&=
n(1+\rho)^2\bar{\sigma}^2.
\label{eq:proof_noise_variance}
\end{align}

Substituting \eqref{eq:proof_direction_second_moment} and
\eqref{eq:proof_noise_variance} into
\eqref{eq:proof_after_young} gives
\begin{align}
\mathbb E
\left[
\mathcal L_\mu(\mathbf w^{t+1})
\,\middle|\,
\mathcal H_t
\right]
&\leq
\mathcal L_\mu(\mathbf w^t)
-
\frac{1-\kappa\mu}{4}
\gamma_t
\|\mathbf s^t\|^2
+
\frac{\rho\gamma_t}{\mu}
\varepsilon_{\Psi,t}
\nonumber\\
&\quad
+
\left(
\frac{\gamma_t}{1-\kappa\mu}
+
\frac{\gamma_t^2}{\mu}
\right)
\|\mathbf d^t\|^2
+
\frac{\gamma_t^2}{\mu}
\|\mathbf V^t\|^2
\nonumber\\
&\quad
+
\frac{n(1+\rho)^2\gamma_t^2}{2\mu}
\bar{\sigma}^2.
\label{eq:proof_before_total_expectation}
\end{align}

Finally, taking the total expectation, using $\mathbb E\|\mathbf V^t\|^2\leq H^2$, and recalling that
$\mathbf s^t=\nabla\mathcal L_\mu(\mathbf w^t)$ proves
\eqref{eq:moreau_one_step_descent}.
\vspace{-0.3cm}
\subsection{Proof of Lemma~\ref{lem:lemma5}}
\label{app:lem5}
Let $
\bar d_{\mathcal G}
:=
\max_{i}
\sum_{j\in\mathcal N_i}a_{ij} $
denote the maximum weighted degree. Using the communication-copy
error $e_i^t=w_i^t-\widehat w_i^t$, the common-copy property of the
broadcast communication mechanism gives
\vspace{-0.1cm}
\begin{equation}
\mathcal E^t
=
\frac{1}{n}
\sum_{i=1}^n
d_i^{\mathcal G}\|e_i^t\|^2,
\qquad
d_i^{\mathcal G}
:=
\sum_{j:i\in\mathcal N_j}a_{ji}
\leq
\bar d_{\mathcal G}.
\label{eq:residual_weighted_copy_error}
\end{equation}

After the local update, the innovation relative to the most recently
reconstructed copy is
\begin{equation}
s_i^{t+1}
=
w_i^{t+1}-\widehat w_i^t
=
e_i^t-\gamma_t h_i^t.
\label{eq:proof_copy_innovation}
\end{equation}

Suppose first that the event-trigger condition is satisfied. Client
$i$ transmits
$c_i^{t}
=
\mathcal C(s_i^{t+1}),
$
and hence
$e_i^{t+1}
=
s_i^{t+1}-\mathcal C(s_i^{t+1})$.
Assumption~\ref{ass:compression} therefore gives
\begin{equation}
\mathbb E
\left[
\|e_i^{t+1}\|^2
\,\middle|\,
s_i^{t+1}
\right]
\leq
(1-\omega)\|s_i^{t+1}\|^2.
\label{eq:proof_compression_contraction}
\end{equation}
For $0<\omega<1$, applying Young's inequality with
$\eta=\omega/[2(1-\omega)]$ yields
\begin{align}
(1-\omega)
\|e_i^t-\gamma_t h_i^t\|^2
&\leq
\left(1-\frac{\omega}{2}\right)
\|e_i^t\|^2
+
\frac{2\gamma_t^2}{\omega}
\|h_i^t\|^2.
\label{eq:proof_triggered_residual}
\end{align}
The same inequality holds trivially when $\omega=1$. Consequently,
in the triggered case,
\begin{align}
\mathbb E\|e_i^{t+1}\|^2
&\leq
\left(1-\frac{\omega}{2}\right)
\mathbb E\|e_i^t\|^2
+
\frac{2\gamma_t^2}{\omega}
\mathbb E\|h_i^t\|^2.
\label{eq:proof_triggered_case}
\end{align}

If the event-trigger condition is not satisfied, no message is sent,
so that
$e_i^{t+1}=s_i^{t+1}$.
The trigger rule implies $\|e_i^{t+1}\|^2
=
\|s_i^{t+1}\|^2
<
\vartheta_t^2$. Since the remaining terms below are nonnegative, this further gives
\begin{align}
\|e_i^{t+1}\|^2
&\leq
\left(1-\frac{\omega}{2}\right)
\|e_i^t\|^2
+
\frac{2\gamma_t^2}{\omega}
\|h_i^t\|^2
+
\vartheta_t^2.
\label{eq:proof_nontriggered_case}
\end{align}

Combining the triggered and non-triggered cases, multiplying by
$d_i^{\mathcal G}$, summing over the clients, and using
$d_i^{\mathcal G}\leq\bar d_{\mathcal G}$ gives
\begin{align}
\mathbb E[\mathcal E^{t+1}]
&\leq
\left(1-\frac{\omega}{2}\right)
\mathbb E[\mathcal E^t]
+
\frac{2\bar d_{\mathcal G}}{\omega}
\gamma_t^2
\frac{1}{n}
\mathbb E\|\mathbf h^t\|^2
+
\bar d_{\mathcal G}\vartheta_t^2.
\label{eq:communication_residual_raw}
\end{align}
Thus, \eqref{eq:communication_residual1} holds with
$
\chi_0=\frac{1}{2},
\quad
\chi_1=\frac{2\bar d_{\mathcal G}}{\omega},
\quad
\chi_2=\bar d_{\mathcal G}$.

It remains to relate the graph-direction approximation error to the
communication residual. From
\eqref{eq:graph_direction_error_expanded} and the weighted
Cauchy-Schwarz inequality,
\begin{align}
\|d_i^t\|^2
&=
\alpha^2
\left\|
\sum_{j\in\mathcal N_i}
a_{ij}
\left(
w_j^t-\widehat w_{j\to i}^t
\right)
\right\|^2
\nonumber\\
&\leq
\alpha^2
\left(
\sum_{j\in\mathcal N_i}a_{ij}
\right)
\sum_{j\in\mathcal N_i}
a_{ij}
\left\|
w_j^t-\widehat w_{j\to i}^t
\right\|^2
\nonumber\\
&\leq
\alpha^2\bar d_{\mathcal G}
\sum_{j\in\mathcal N_i}
a_{ij}
\left\|
w_j^t-\widehat w_{j\to i}^t
\right\|^2.
\end{align}
Summing over $i$ yields
$\|\mathbf d^t\|^2
\leq
n\alpha^2\bar d_{\mathcal G}\mathcal E^t$.
Therefore, \eqref{eq:graph_error_residual_relation} holds with $
K_{\mathcal G} =
\alpha^2\bar d_{\mathcal G}$.
This completes the proof.
\hfill\(\square\)

\subsection{Proof of Lemma~\ref{lem:lemma6}}
\label{app:lem6}

From \eqref{eq:closed_residual_recursion},
\begin{align}
\mathbb E[\mathcal E^{t+1}]
&\leq
\left(
1-\chi_0\omega
+
2\chi_1K_{\mathcal G}\gamma_t^2
\right)
\mathbb E[\mathcal E^t]
+
\chi_1B_h\gamma_t^2
+
\chi_2\vartheta_t^2.\nonumber
\label{eq:proof_closed_residual}
\end{align}
Since $\gamma_t^2=\gamma^2/T\leq\gamma^2$ and
$2\chi_1K_{\mathcal G}\gamma^2\leq\chi_0\omega/2$, we have
\[
1-\chi_0\omega
+
2\chi_1K_{\mathcal G}\gamma_t^2
\leq
1-\frac{\chi_0\omega}{2}.
\]
Therefore,
\begin{align}
\mathbb E[\mathcal E^{t+1}]
&\leq
\left(
1-\frac{\chi_0\omega}{2}
\right)
\mathbb E[\mathcal E^t]
+
\chi_1B_h\gamma_t^2
+
\chi_2\vartheta_t^2.
\label{eq:proof_contractive_residual}
\end{align}

Rearranging and summing over $t=0,\ldots,T-1$ gives
\begin{align}
\frac{\chi_0\omega}{2}
\sum_{t=0}^{T-1}
\mathbb E[\mathcal E^t]
&\leq
\mathbb E[\mathcal E^0]
-
\mathbb E[\mathcal E^T]
+
\chi_1B_h
\sum_{t=0}^{T-1}\gamma_t^2
+
\chi_2
\sum_{t=0}^{T-1}\vartheta_t^2
\nonumber\\
&\leq
\mathbb E[\mathcal E^0]
+
\chi_1B_h\gamma^2
+
\chi_2\vartheta_0^2(1+\log T),
\label{eq:proof_residual_sum}
\end{align}
where we used
$\sum_{t=0}^{T-1}\gamma_t^2=\gamma^2$ and
$\sum_{t=0}^{T-1}(t+1)^{-1}\leq1+\log T$.
Dividing \eqref{eq:proof_residual_sum} by
$\chi_0\omega T/2$ yields
\begin{align}
\frac{1}{T}
\sum_{t=0}^{T-1}
\mathbb E[\mathcal E^t]
&\leq
\frac{2}{\chi_0\omega T}
\left[
\mathbb E[\mathcal E^0]
+
\chi_1B_h\gamma^2
+
\chi_2\vartheta_0^2(1+\log T)
\right],
\end{align}
which proves \eqref{eq:communication_residual_average}.
\hfill\(\square\)
\vspace{-0.1cm}
\subsection{Proof of Theorem~\ref{thm:main}}
\label{app:proof_main}
Summing the one-step inequality
\eqref{eq:moreau_one_step_descent} over
$t=0,\ldots,T-1$ and using
$\mathcal L_\mu(\mathbf w^T)\geq\mathcal L^\star$ gives
\begin{align}
\frac{1-\kappa\mu}{4}
\sum_{t=0}^{T-1}
\gamma_t
\mathbb E
\left[
\left\|
\nabla\mathcal L_\mu(\mathbf w^t)
\right\|^2
\right]
&\leq
\Delta_\mu^0
+
\frac{\rho}{\mu}
\sum_{t=0}^{T-1}
\gamma_t\mathbb E[\varepsilon_{\Psi,t}]
\nonumber\\
&\quad
+
\sum_{t=0}^{T-1}
\left(
\frac{\gamma_t}{1-\kappa\mu}
+
\frac{\gamma_t^2}{\mu}
\right)
\mathbb E\|\mathbf d^t\|^2\nonumber\\
&\quad +
\frac{n(1+\rho)^2\bar{\sigma}^2}{2\mu}
\sum_{t=0}^{T-1}\gamma_t^2
\nonumber\\
&\quad
+
\frac{H^2}{\mu}
\sum_{t=0}^{T-1}\gamma_t^2,
\label{eq:proof_main_summed}
\end{align}
where
$\Delta_\mu^0
=\mathcal L_\mu(\mathbf w^0)-\mathcal L^\star$.

Since $\gamma_t=\gamma/\sqrt T$, dividing
\eqref{eq:proof_main_summed} by
$nT\gamma_t(1-\kappa\mu)/4$ and using
$\mathcal S_\mu^t
=n^{-1}\|\nabla\mathcal L_\mu(\mathbf w^t)\|^2$
yields
\begin{align}
\frac{1}{T}
\sum_{t=0}^{T-1}
\mathbb E[\mathcal S_\mu^t]
&\leq
\frac{4\Delta_\mu^0}
{n(1-\kappa\mu)\gamma\sqrt T}
+
\frac{4\rho}
{n\mu(1-\kappa\mu)}
\frac{1}{T}
\sum_{t=0}^{T-1}
\mathbb E[\varepsilon_{\Psi,t}]
\nonumber\\
&\quad
+
\frac{4}
{n(1-\kappa\mu)}
\left(
\frac{1}{1-\kappa\mu}
+
\frac{\gamma}{\mu\sqrt T}
\right)
\frac{1}{T}
\sum_{t=0}^{T-1}
\mathbb E\|\mathbf d^t\|^2
\nonumber\\
&\quad
+
\frac{4\gamma H^2}
{n\mu(1-\kappa\mu)\sqrt T}
+
\frac{2\gamma(1+\rho)^2\bar{\sigma}^2}
{\mu(1-\kappa\mu)\sqrt T}.
\label{eq:proof_main_pre_bounds}
\end{align}

By the tower property and
Assumption~\ref{ass:selection_accuracy},
\begin{equation}
\frac{1}{T}
\sum_{t=0}^{T-1}
\mathbb E[\varepsilon_{\Psi,t}]
\leq
\frac{B_\Psi}{\sqrt T}.
\label{eq:proof_fairness_average}
\end{equation}
Moreover, Lemma~\ref{lem:lemma5} gives
$\|\mathbf d^t\|^2
\leq nK_{\mathcal G}\mathcal E^t$, and hence
\begin{equation}
\frac{1}{T}
\sum_{t=0}^{T-1}
\mathbb E\|\mathbf d^t\|^2
\leq
nK_{\mathcal G}
\frac{1}{T}
\sum_{t=0}^{T-1}
\mathbb E[\mathcal E^t].
\label{eq:proof_graph_error_average}
\end{equation}

Substituting \eqref{eq:proof_fairness_average} and
\eqref{eq:proof_graph_error_average} into
\eqref{eq:proof_main_pre_bounds}, and then applying
Lemma~\ref{lem:lemma6}, yields
\begin{align}
\frac{1}{T}
\sum_{t=0}^{T-1}
\mathbb E[\mathcal S_\mu^t]
&\leq
\frac{4\Delta_\mu^0}
{n(1-\kappa\mu)\gamma\sqrt T}
+
\frac{4\rho B_\Psi}
{n\mu(1-\kappa\mu)\sqrt T}
\nonumber\\
&\quad
+
\frac{4\gamma H^2}
{n\mu(1-\kappa\mu)\sqrt T}
+
\frac{2\gamma(1+\rho)^2\bar{\sigma}^2}
{\mu(1-\kappa\mu)\sqrt T}
\nonumber\\
&\quad
+
\frac{8K_{\mathcal G}}
{\chi_0\omega T(1-\kappa\mu)}
\left(
\frac{1}{1-\kappa\mu}
+
\frac{\gamma}{\mu\sqrt T}
\right)
\nonumber\\
&\qquad\qquad\times
\left[
\mathbb E[\mathcal E^0]
+
\chi_1B_h\gamma^2
+
\chi_2\vartheta_0^2(1+\log T)
\right].\nonumber
\end{align}
This is precisely \eqref{eq:main_convergence_bound}. Since all
problem-dependent constants are independent of $T$, the first four
terms are $\mathcal O(T^{-1/2})$, while the communication term is
$\mathcal O((1+\log T)/T)$. Therefore,
\[
\frac{1}{T}
\sum_{t=0}^{T-1}
\mathbb E[\mathcal S_\mu^t]
=
\mathcal O(T^{-1/2})
+
\mathcal O\left(\frac{\log T}{T}\right).
\]
This completes the proof.
\hfill\(\square\)
\bibliographystyle{IEEEtran}
\bibliography{strings}
\end{document}